\documentclass[twoside,11pt]{article}

\usepackage{blindtext}
\usepackage{amsthm}
\usepackage{amsmath}
\usepackage{jmlr2e}

\newcommand{\cls}{\mathcal{C}}
\newcommand{\con}{\mathit{conf}}

\newcommand{\zone}{\mathcal{Z}}

\newcommand{\model}{\mathcal{M}}
\newcommand{\sep}{\overline{\mathcal{S}}^\model}
\newcommand{\stab}{\underline{\mathcal{S}}^\model}

\newcommand{\tr}{\mathit{Tr}}
\newcommand{\ts}{\mathit{Ts}}
\newcommand{\vs}{\mathit{Vs}}

\newcommand{\friend}{\mathit{F}_\model}
\newcommand{\nfriend}{\overline{\mathit{F}}_\model}

\newcommand{\dis}{\mathit{d}}
\newcommand{\Dis}{\mathit{D}}

\usepackage{lastpage}
\jmlrheading{23}{2023}{1-\pageref{LastPage}}{1/23; Revised -}{-}{-}{Gabriella Chouraqui et al}

\ShortHeadings{Uncertainty Estimation based on Geometric Separation}{Chouraqui et al.}
\firstpageno{1}

\usepackage{enumitem}
\newlist{steps}{enumerate}{1}
\setlist[steps, 1]{label = Step \arabic*:}

\usepackage{times}
\usepackage{soul}
\usepackage{caption}
\usepackage[utf8]{inputenc}
\usepackage{amsfonts}
\usepackage{listings}

\usepackage{graphicx}
\usepackage[capitalise,noabbrev]{cleveref}
\newtheorem{defn}{Definition}
\newtheorem{prop}{Proposition}

\newtheorem{lema}{Lemma}
\usepackage{ dsfont }
\usepackage{booktabs}
\usepackage[nomargin,inline,final]{fixme}
\fxusetheme{color}
\fxuseenvlayout{color}
\usepackage{multirow}
\usepackage{textcomp}

\usepackage[section]{placeins}
\usepackage{rotating}
\usepackage[table, svgnames, dvipsnames]{xcolor}
\usepackage{hyperref}
\usepackage{caption}
\usepackage{subcaption}
\hypersetup{
  colorlinks=true,
  linkcolor=blue
  citecolor=violet,
  linkcolor=red,
  urlcolor=blue}
  
\definecolor{ao}{rgb}{0.0, 0.0, 1.0}
\definecolor{amethyst}{rgb}{0.6, 0.4, 0.8}
\definecolor{Green}{rgb}{0.55, 0.71, 0.0}
\FXRegisterAuthor{lc}{alc}{\color{magenta}[Liron]}
\FXRegisterAuthor{ge}{age}{\color{ao}[Gil]}
\FXRegisterAuthor{gc}{agc}{\color{amethyst}[Gabriella]}
\FXRegisterAuthor{ll}{all}{\color{Green}[Liel]}

\begin{document}

\title{Uncertainty Estimation based on Geometric Separation}
 \author{\name Gabriella Chouraqui \email chouraga@post.bgu.ac.il \\
      \addr Department of Computer Science\\
      Ben-Gurion University\\
       Israel
       \AND
       \name Liron Cohen \email cliron@bgu.ac.il \\
       \addr Department of Computer Science\\
             Ben-Gurion University\\
       Israel
        \AND
        \name Gil Einziger \email gilein@bgu.ac.il \\
       \addr Department of Computer Science\\
      Ben-Gurion University\\
       Israel
               \AND
        \name Liel Leman \email Leman@post.bgu.ac.il \\
       \addr Department of Computer Science\\
      Ben-Gurion University\\
       Israel
       }

\editor{}
       
\maketitle

\begin{abstract}
In machine learning, accurately predicting the probability that a specific input is correct is crucial for risk management. This process, known as uncertainty (or confidence) estimation, is particularly important in mission-critical applications such as autonomous driving. 
In this work, we put forward a novel geometric-based approach for improving uncertainty estimations in machine learning models. 
Our approach involves using the geometric distance of the current input from existing training inputs as a signal for estimating uncertainty, and then calibrating this signal using standard post-hoc techniques. 
We demonstrate that our method leads to more accurate uncertainty estimations than recently proposed approaches through extensive evaluation on a variety of datasets and models. 
Additionally, we optimize our approach so that it can be implemented on large datasets in near real-time applications, making it suitable for time-sensitive scenarios.

\end{abstract}

\begin{keywords}
  uncertainty estimation, geometric separation, calibration, confidence evaluation.
\end{keywords}

\section{Introduction}





Machine learning models, such as neural networks, random forests, and gradient boosted trees, are widely used in various fields, including computer vision and transportation, and are transforming the field of computer science~\cite{Survey1,transport}. However, the probabilistic nature of classifications made by these models means that misclassifications are inevitable. As a result, estimating the uncertainty for a particular input is a crucial challenge in machine learning. In fact, many machine learning models have some built-in measure of confidence that is often provided to the user for risk management purposes. The field of \emph{uncertainty calibration} aims to improve the accuracy of the confidence estimates made by machine learning models~\cite {pmlr-v70-guo17a}.


 %
 Confidence evaluation, or the model's prediction of its success rate on a specific input, is a crucial aspect of mission-critical machine learning applications, as it provides a realistic estimate of the probability of success for a classification and enables informed decisions about the \emph{current} situation. Even a highly accurate model may encounter an unexpected situation, which can be communicated to the user through confidence estimation. For example, consider an autonomous vehicle using a model to identify and classify traffic signs. The model is very accurate, and in most cases, its classifications are correct with high confidence. However, one day, it encounters a traffic sign that is obscured, e.g.,  by heavy vegetation. In this case, the model's classification is likely to be incorrect.
 Estimating confidence, or uncertainty, is a crucial tool for assessing unavoidable risks, allowing system designers to address these risks more effectively and potentially avoid unexpected and catastrophic consequences. For example, our autonomous vehicle may reduce its speed and activate additional sensors until it reaches higher confidence. 
 Therefore, all popular machine learning models have mechanisms for determining confidence that can be calibrated to maximize the quality of confidence estimates~\cite{Niculescu2005Predicting,CNNCalibration,Ana2019Verified}, and there is ongoing research to calibrate models more effectively and enable more reliable applications~\cite{Survey2,Sun2007}. 

Existing calibration methods can be divided into two categories: post-hoc methods that perform a transformation that maps the raw outputs of classifiers to their expected probabilities~\cite{NEURIPS2019_8ca01ea9,pmlr-v70-guo17a,gupta2021distribution}, and ad-hoc methods that adapt the training process to produce better calibrated models~\cite{ThulasidasanCBB19,pmlr-v97-hendrycks19a}. Post-hoc calibration methods are easier to apply because they do not change the model and do not require retraining. However, ad-hoc methods may lead to better model training in the first place and more reliable models. With the success of both approaches, recent research has focused on using ensemble methods whose estimates are a weighted average of multiple calibration methods~\cite{Ashukha2020Pitfalls,pmlr-v161-ma21a,pmlr-v119-zhang20k,naeini2016binary,naeini2015obtaining}.
Another recent line of work attempts to further refine the uncertainty estimations by refining the grouping of confidence estimations, e.g.,~\cite{ByoundCalibration,grouppingConfidence}.


%
In principle, post-hoc calibration can be viewed as cleaning up a signal, namely the model's original confidence estimate. Interestingly, if we follow this logic, it is clear that the maximal attainable benefit lies in the quality of the signal. To see this, consider a model that plots the same confidence for all inputs. In this case, the best result that can be achieved is to set that confidence to the model's average accuracy over all inputs. Therefore, finding better signals to calibrate is a promising direction for research.

In this work, we introduce a novel approach for improving uncertainty estimates in machine learning models \emph{using geometry}. 
%
We first provide an algorithm for calculating the maximal geometric \emph{separation} of an input.
However, calculating the geometric separation of an input requires evaluating the whole space of training inputs, making it a computationally expensive method that is not always feasible.  
Therefore, we suggest multiple methods to accelerate the process, including a lightweight approximation called \emph{fast-separation} and several data reduction methods that shorten the geometric calculation. 


We demonstrate that using our geometric-based method, combined with a standard calibration method, leads to more accurate confidence estimations than calibrating the model's original signal across different models and datasets. 
Even more, our approach yields better estimation even when compared to state-of-the-art calibration methods~\cite{Ana2019Verified,gupta2021distribution,pmlr-v70-guo17a,pmlr-v119-zhang20k,naeini2015obtaining,kull2017beta}. 
Additionally, we show that our approach can be implemented in near real-time on a variety of datasets through the use of multiple levels of approximation and optimization. This is particularly useful for practical applications that require rapid decision-making, such as autonomous driving.
The entire code is available at our Github~\cite{Code}.

\section{Related Work}



As mentioned above, uncertainty calibration is about estimating the model's success probability of classifying a given example. Post-hoc calibration methods apply some transformation to the model's confidence (without changing the model) such transformations include Beta calibration (Beta)~\cite{kull2017beta}, Platt scaling (Platt)~\cite{platt1999probabilistic}, Temperature Scaling (TS)~\cite{pmlr-v70-guo17a,NEURIPS2019_8ca01ea9}, Ensemble Temperature Scaling (ETS)~\cite{pmlr-v119-zhang20k}, and cubic spline~\cite{gupta2021distribution}.
In brief, these methods are limited by the best learnable mapping between the model's confidence estimations, and the actual confidence. That is, post-hoc calibration map each confidence value to another calibrated value whereas our method introduces a new signal that can be calibrated just like the model's original signal.  Another work that uses a geometric distance in this context is \cite{Dalitz09Reject}. There, the confidence score is computed directly from the geometric distance, while we first fit a function on a subset of the data to learn the specific behavior of the dataset and model. Moreover, the work in~\cite{Dalitz09Reject} only applies to the k-nearest neighbor model, while our method is applicable to all models.

The recently proposed Scaling Binning Calibrator (SBC) of~\cite{Ana2019Verified} uses a fitting function on the confidence values, divides the inputs into bins of equal size, and outputs the function's average in each bin.  Histogram Binning (HB) \cite{gupta2021distribution} uses a similar idea but divides the inputs into uniform-mass (rather than equal-size) bins.
Interestingly, while most post-hoc calibration methods are model agnostic, recent methods have begun to look at a neural network non-probabilistic output called logits~(before applying softmax)~\cite{CNNCalibration,Ding2020,wenger2019}. Thus, some new post-hoc calibration methods apply only to neural networks.

Ensemble methods are similar to post-hoc calibration methods as they do not change the model, but they consider multiple signals to determine the model's confidence~\cite{Ashukha2020Pitfalls,pmlr-v161-ma21a}. 
For example, Bayesian Binning into 
Quantiles (BBQ)~\cite{naeini2015obtaining} is an extension of HB that uses multiple histogram binning models with different bin numbers, and partitions then outputs scores  according to Bayesian averaging.
The same methodology of Bayesian averaging is applied in Ensemble of Near Isotonic Regression~\cite{naeini2016binary},
but instead of histogram binning, they use nearly isotonic regression models.


Ad-hoc calibration is about training models in new manners aimed to yield better uncertainty estimations. Important techniques in this category include mixup training~\cite{ThulasidasanCBB19}, pre-training~\cite{pmlr-v97-hendrycks19a},  
label-smoothing~\cite{NEURIPS2019_f1748d6b}, data augmentation \cite{Ashukha2020Pitfalls},  self-supervised learning~\cite{NEURIPS2019_a2b15837}, Bayesian approximation (MC-dropout)~\cite{pmlr-v48-gal16,NIPS2017_84ddfb34},  Deep Ensemble (DE)~\cite{DeepEnsembles}, Snapshot
Ensemble~\cite{SnapshotEnsemble}, Fast Geometric Ensembling (FGE)~\cite{FastEnsembling}, and SWA-Gaussian (SWAG)~\cite{SWAG}. 
A notable approach is to use geometric distances in the loss function while training the model~\cite{Xing2020distance}. The authors work with a representation space that maximizes intra-class distances, minimizes inter-class distances, and uses the distances to estimate the confidence. 
Ad-hoc calibration is perhaps the best approach in public as it tackles the core of models' calibration directly. However, because it offers specific training methods, it is of less use to large and already trained models, and the impact of each workshop is limited to a specific model type (e.g., DNNs in~\cite{FastEnsembling}). In comparison, post-hoc and ensemble methods (and our own method) often work for numerous models.

Our geometric method is largely inspired by the approach of robustness proving in machine learning models. In this field, formal methods are used to prove that specific inputs are robust to small adversarial perturbations. That is, we formally prove that all images in a certain geometric radius around a specific train-set image receive the same classification~\cite{mooly, KatzBDJK17,marta,Gehr2018AISA,Ehlers17,DBLP:conf/aaai/EinzigerGSS19}. These works 
rely on formal methods produced in an offline manner and thus apply only to training set inputs (known apriori). Whereas confidence estimation reasons about the current input. However, the underlying intuition, i.e., that geometrically similar inputs should be classified in the same manner is also common to our work. 

%
Indeed, our work shows that geometric properties of the inputs can help us quantify the uncertainty in certain inputs and that, in general, inputs that are less geometrically separated and are 'on the edge' between multiple classifications are more error-prone than inputs that are highly separated from other classes. Thus our work reinforces the intuition behind applying formal methods to prove robustness and supports the intuition that more robust training models would be more dependable. 


\section{Geometric Separation}
In this section, we define a geometric separation measure that reasons about the distance of a given input from other inputs with different classifications. Our end goal is to use this measure to provide confidence estimations. Formally, a model receives a data input, $x$, and outputs the pair $\langle\cls(x),\con(x)\rangle$, where $\cls(x)$ is the model's classification of $x$ and $\con(x)$ reflects the probability that the classification is correct.  We estimate the environment around $x$ where inputs are closer to inputs of certain classifications over the others. Our work assumes that the inputs are normalized, and thus these distances carry the same significance between the different inputs. 

In~\cref{sec:sep}, we define geometric separation and provide an algorithm to calculate it. Our evaluation shows that geometric separation produces a valuable signal that improves confidence estimations. However, calculating geometric separation is too cumbersome for real-time systems, so we suggest a lightweight approximation in~\Cref{sec:stab}. 
Finally,~\cref{sec:conf} explains how we use the geometric signal to derive $\con(x)$. That is, mapping a real number corresponding to the geometric separation to a number in $[0,1]$ corresponding to the confidence ratio. 


\subsection{Separation Measure}
\label{sec:sep}
We look at the displacement of $x$ compared to nearby data inputs within the training set. Intuitively, when $x$ is close to other inputs in $\cls(x)$ (i.e., inputs with the same classification as $x$) and is far from inputs with other classifications, then the model is correct with a high probability, implying that $\con(x)$ should be high. On the other hand, when there are training inputs with a different classification close to $x$,  we estimate that  $\cls(x)$ is more likely to be incorrect. 

Below we provide definitions that allow us to formalize this intuitive account. In what follows, we consider a model $\model$ to consist of a machine learning model (e.g., a gradient boosted tree or a neural network), along with a labeled train set, $\tr$, used to generate the model. 
We use an implicit notion of distance and denote by $\dis(x,y)$ the distance between inputs $x$ and $y$, and by $\Dis(x,A)$ the distance between the input $x$ and the set $A$ (i.e., the minimal distance between $x$ and the inputs in $A$).

\begin{defn}[Safe and Dangerous inputs]
\label{def:Tr_C(x)}
Let $\model$ be a model.  
For an input $x$ in the sample space we define:
\[\friend(x) :=\{x'\in \tr : \cls(x')=\cls(x)\}.\]
We denote by $\nfriend(x)$ the set $\tr \setminus \friend(x)$.
An input $x\in \mathcal{X}$ is labeled as \emph{safe} if $D(x,\friend(x)) < D(x,\nfriend(x))$, and it is labeled as \emph{dangerous} otherwise.
\end{defn}

\begin{defn}[Zones]
\label{def:zone}
Let $x$ be a safe (dangerous) input. 
A \emph{zone} for $x$, denoted $z_x$, is such that for any input $y$, if $d(x,y)<z_x$, then $\Dis(y,\friend(x))<\Dis(y,\nfriend(x))$ ($\Dis(y,\friend(x))\geq\Dis(y,\nfriend(x))$). 
For each $x$ we denote the maximal such zone by $\zone(x)$.
\end{defn}

In other words, a zone of a safe (dangerous) input $x$ is a  radius around $x$ such that all inputs in this ball are  closer to an input in $\friend(x)$ ($\nfriend(x)$) than to any input in $\nfriend(x)$ ($\friend(x)$). $\zone(x)$ is the \emph{maximal} zone attainable of $x$. 

\begin{figure}[t!]
    \centering
    \includegraphics[width=0.5\textwidth]{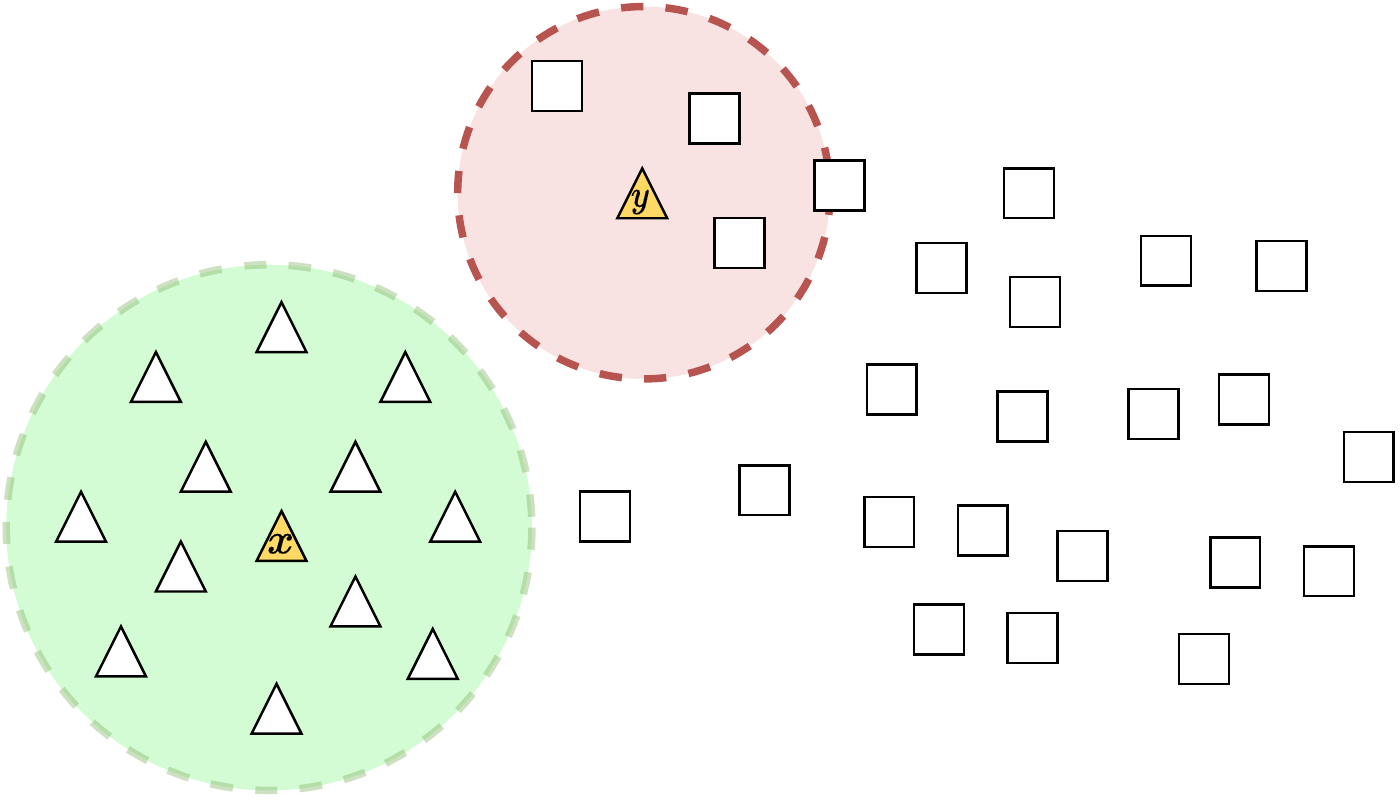}
    \caption{Geometric representation of safe and dangerous inputs, maximal zones, and separation values.
    The various classifications are illustrated via different shapes, and the safe (danger) zones of x (y) are illustrated via green (red) circles.}
    \label{fig:Sep-circles}
\end{figure}
\Cref{fig:Sep-circles} provides a geometric illustration of the safe and danger zones of a given input and of the separation values. For illustration purposes, the figure uses the $L_2$ norm with two dimensions, whereas our data usually includes many more dimensions. For example, a $30\times30$ traffic sign image will have 900 dimensions.
In the figure, the shapes represent the classification of training set inputs. In yellow, we see a new input ($x$ on the left-hand-side and $y$ on the right-hand-side) which the model classifies as a triangle. 
$x$ is a safe input because it is closer to other triangles in the training set than it is to the squares.
The green highlighted ball represents its maximal zone.  
The input $y$ is dangerous because the closest training set input is a square. 
The red highlighted ball represents its maximal zone which dually represents how far we need to distance ourselves from $y$ so that inputs classified as triangles may become closer than other inputs.  

\begin{defn}[Separation]
\label{def:sep}
The separation of a data input $x$ with respect to the model $\model$ is  $\zone(x)$ when  $x$ is a safe input, and $-1\cdot \zone(x)$
when $x$ is a dangerous input. 
\end{defn}

That is, the separation of $x$ encapsulates the maximal zone for $x$ (provided by the absolute value) together with an indication of whether the input is safe or dangerous (provided by the sign).
The separation of $x$ depends only on the classification of $x$ by the model and the train set. This is because our definition partitions the inputs in $\tr$ into two sets: one with $\cls(x)$, $\friend(x)$, and one with all other classifications, $\nfriend(x)$. These sets vary between models only when they disagree on the classification of $x$.
Note that $x$'s for which the distance from $\friend(x)$ equals the distance from $\nfriend(x)$ are considered dangerous inputs, and their separation measure will be zero.

As mentioned,~\Cref{def:zone} and~\Cref{def:sep} use an implicit notion of distance and can accept any distance metric (e.g., $L_1, L_2$ or $L_{\infty}$). However, throughout this work, we use $L_2$ as it is a standard measure for safety features in adversarial machine learning~\cite{Robustness-Moosavi}, in addition to it being easy to illustrate and intuitive to understand.
Moreso, as our work targets real-time confidence estimations using $L_2$ allows us to leverage standard and well-optimized libraries. Accordingly, all our definitions and calculations assume the $L_2$ metrics (Euclidean distances). 
Nevertheless, ~\Cref{sec:metrics} shows that other metrics are also feasible. 



Next, we provide a formula for calculating the separation of a given input $x$ within the $L_2$ distance metric.

\begin{defn}
\label{def:practicalsep}
Given a model $\model$ and an input $x$, define:
\[\sep(x)= \min_{x'' \in \nfriend(x)} \max_{x' \in \friend(x)} \frac{\dis^2(x,x'') - \dis^2(x,x') }{2\dis(x',x'')}\]
\end{defn}

\begin{lema}\label{lem:sep}
Let $x, x', x'' \in \mathbb{R}^n$ be inputs such that $\dis(x,x')<\dis(x,x'')$. 
The maximal distance $M(x,x',x'')$ for which if $y\in \mathbb{R}^n$ such that $\dis(x,y)<M(x,x',x'')$, then 
$\dis(y,x')< \dis(y,x'')$ is \[\frac{\dis^2(x,x'')-\dis^2(x,x')}{2\dis(x',x'')}.\]
\end{lema}
\begin{proof}
Since any three points in space define a plane we focus on the plane defined by these three points. 
\begin{figure}[h!]
    \centering
    \includegraphics[width=0.35\textwidth]{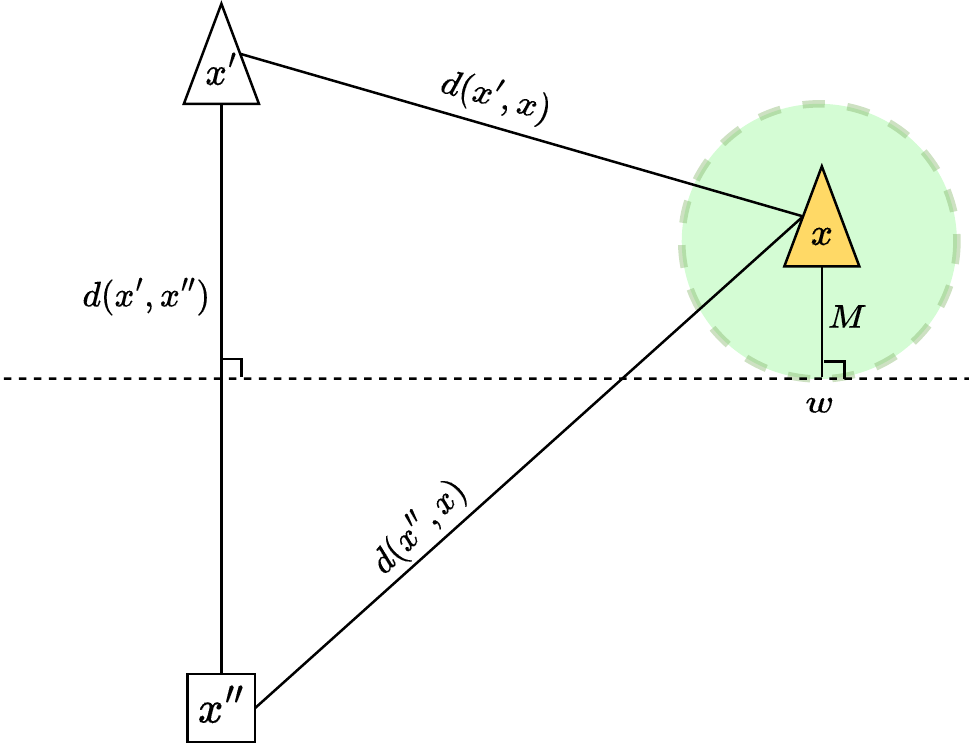}
    \caption{Illustration of the proof of~\Cref{lem:sep}}
    \label{fig:Sep-proof}
\end{figure}

\Cref{fig:Sep-proof} demonstrates a geometric positioning of the points and the main constructions in the proof.
The perpendicular bisector to the line between $x'$ and $x''$ divides the plane into two parts: one in which all the points are closer to $x''$ than to $x'$ (the lower part in the figure) and one in which all the points are closer to $x'$ than to $x''$ (the upper part in the figure).
Our goal is thus to establish the distance between $x$ and the lower part of the plane. Hence, $M(x,x',x'')$ amounts to the distance from $x$ to the perpendicular bisector to the line between $x'$ and $x''$.
Using trigonometric calculations, it is straightforward to verify that indeed 
$$M(x,x',x'')=\frac{\dis^2(x,x'')-\dis^2(x,x')}{2\dis(x',x'')}.$$\qedhere
\end{proof}

\begin{prop} \label{prop:sep}
$\sep(x)$ is the separation of $x$ with respect to the model $\model$ (in~\Cref{def:sep}).  
\end{prop}
\begin{proof}
Let $x$ be a safe input, and $y$ be an input such that: \[\dis(x,y)< \min_{x'' \in \nfriend(x)} \max_{x' \in \friend(x)} \frac{\dis^2(x,x'') - \dis^2(x,x') }{2\dis(x',x'')}.\]
We first show that $y$ is closer to $\friend(x)$ than to $\nfriend(x)$.
Let $z''\in \nfriend(x)$,
it suffices to show that there exist $z' \in \friend(x)$ such that $\dis(y,z')<\dis(y,z'')$. 
%
Notice that:
\[\dis(x,y)<\max_{x' \in \friend(x)} \frac{\dis^2(x,z'') - \dis^2(x,x') }{2\dis(x',z'')}. \]
Therefore, there exist a $z'\in \friend(x)$ for which:
\[\dis(x,y)< \frac{\dis^2(x,z'') - \dis^2(x,z') }{2\dis(z',z'')}\]
Thus, since $x$ is a safe input,  using~\Cref{lem:sep}, we conclude that $\dis(y,z')< \dis(y,z'')$.
%
The proof follows similar arguments for dangerous inputs,  taking the distances as $-\sep$ and flipping the inequalities.

To show the maximality, observe that the intersection point marked by $w$ in~\Cref{fig:Sep-proof}, which is at distance $\sep(x)$ from $x$,  can be easily shown to be of equal distances from $\friend(x)$ and $\nfriend(x)$.
%
%
\end{proof}


While separation provides the maximal zone, it is expensive to calculate. As can be seen  in~\Cref{def:practicalsep}, to estimate the separation of one specific input, we go over many triplets of inputs. The exact amount is unbounded and depends on the dataset.  Thus, separation is infeasible to compute in near real-time. Therefore, when time or computation resources are limited, we require a different and computationally simpler notion. 
Accordingly, the following section provides an efficient approximation of the separation measure.

\subsection{Fast-Separation Approximation}
\label{sec:stab}
We approximate the separation of a given input using only its distance from $\friend(x)$ and its distance from $\nfriend(x)$. 
This simplification allows us to calculate a zone for any given input, which is not necessarily the maximal one. 
The reliance on these two distances enables a faster calculation since we do not perform an exhaustive search over many triplets of inputs.
In particular, we do not consider the geometric positioning of the inputs that determine the distance from these sets. 
%

\begin{defn}[Fast-Separation] 
\label{def:stab}

Given a model $\model$, the fast-separation of an input $x$, denoted $\stab(x)$, is defined as:
\[\stab(x)=\frac{\Dis(x,\nfriend(x))- \Dis(x,\friend(x))}{2}\]
\end{defn}

Notice that just as is the case for separation, if $x$ is a safe input, its fast-separation value will be strictly positive and non-positive otherwise.

\Cref{fig:Stab_theory} illustrates the notion of fast-separation. In particular, it exemplifies why it only provides an approximation of the more accurate separation measure. It encapsulates a zone that is less than or equal to that of separation. 
Sub-figure  (a) demonstrates a case in which $\stab(x)=\sep(x)$, while sub-figure (b) presents a case where $\sep(x)$ is considerably larger than $\stab(x)$.
    

\begin{figure}[t!]
    \centering
    \subfloat[$\stab(x) = \sep(x) = 0.5$ ]{{\includegraphics[width=0.4\linewidth]{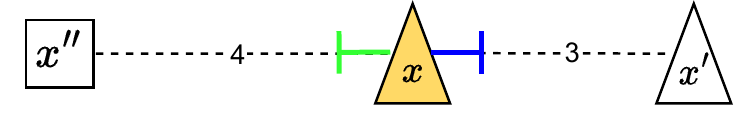} }}
    \centering 
    \subfloat[$0.5=\stab(x)   \neq \sep(x) = 3.5$ ]{{\includegraphics[width=0.45\linewidth]{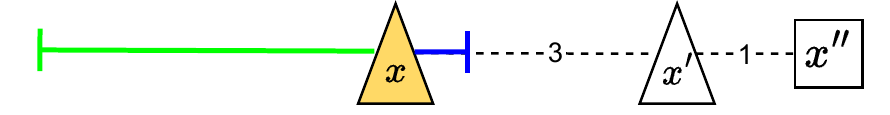} }}

    \caption{Geometric representation of the  induced zones of $\stab$ and $\sep$ for different input alignments.
    $\stab$ is represented by blue arrows and $\sep$ by green arrows.}
    \label{fig:Stab_theory}
\end{figure}

The separation measure defined as the maximal safe zone is applicable to all norms. However, the explicit formula $\sep$, given in~\Cref{def:practicalsep} is only applicable in $L_2$. The following proposition demonstrates that fast separation, $\stab$, calculates a zone that is always contained in the maximal zone for any distance metric. Thus, it approximates the geometric separation for all metrics as the proof only requires the triangle inequality.

\begin{prop}\label{prop:stab_all_l}
For any metric $\ell$, and for any input $x$, $\stab(x)$ (calculated with respect to $\ell$) is a zone of $x$. That is, $|\stab(x)|\leq \zone(x)$. Furthermore, $\stab(x)$ has the same sign as the separation of $x$.

\end{prop}
\begin{proof}
Let $x$ be a safe input, we show that $\stab(x)$ is a zone of $x$.


Let $y$ be a point such that \[\dis(x,y)< \stab(x)= \frac{\Dis(x,\nfriend(x))- \Dis(x,\friend(x))}{2} .\]
We show that $\Dis(y,\friend(x))<\Dis(y,\nfriend(x))$.
Take $z' \in \friend(x)$ and  $z'',w \in \nfriend(x)$  such that $\dis(x,z')=\Dis(x,\friend(x))$,  $\dis(x,z'')=\Dis(x,\nfriend(x))$, 
and $\dis(y,w)=\Dis(y,\nfriend(x))$.
Using the triangle inequality we get:
\begin{gather*}
     \Dis(y,\friend(x))\leq \dis(y,z')\leq \dis(x,z')+ \dis(x,y)\\
    < \dis(x,z')+ \frac{\dis(x,z'')-\dis(x,z')}{2}
    =\frac{\dis(x,z'')+\dis(x,z')}{2}\\
    = \dis(x,z'') - \frac{\dis(x,z'')-\dis(x,z')}{2}
    < \dis(x,z'')-\dis(x,y)\\
    \leq \dis(x,w)-\dis(x,y) \leq \dis(y,w)
    = \Dis(y,\nfriend(x))
\end{gather*}
For dangerous inputs, the proof follows similar arguments, switching $\friend(x)$ and $\nfriend(x)$.

For the sign of $\stab(x)$ it is easy to see that for a safe (dangerous) input, $\stab(x)$ will be positive (negative) and therefore has the same sign as the separation.
\end{proof}

\cref{prop:stab_all_l}  shows that $\stab(x)$ induces a zone that is always smaller than the maximal zone for any distance metric. In the case of $L_2$, we have a formula for calculating the maximal zone ($\sep(x)$), and the following proposition provides an approximation bound.



\begin{prop}\label{prop:bound}

The following holds for any point $x$:
\[|\sep(x)-\stab(x)|\leq \frac{\Dis(x,\friend(x))+ \Dis(x,\nfriend(x))}{2}.\]
\end{prop}

\begin{proof}
We here prove the bound for safe inputs $x$, the proof for dangerous inputs is similar. 
Let $x$ be a safe input. By definition:
\begin{align*}
&|\sep(x)-\stab(x)|= \sep(x)-\stab(x) =  \\
=&\min_{x'' \in \nfriend(x)} \max_{x' \in \friend(x)} \frac{\dis^2(x,x'') - \dis^2(x,x') }{2\dis(x',x'')}\\
&\quad -\frac{\Dis(x,\nfriend(x))- \Dis(x,\friend(x))}{2} 
\end{align*}
Let $z'' \in \nfriend(x)$  be an input such that $\dis(x,z'')=\Dis(x,\nfriend(x))$, and let $z'\in \friend(x)$ be a input for which the  maximum on the expression above is obtained. Then, we have: 
\begin{small}
\begin{align}
&|\sep(x)-\stab(x)| \notag \\
\leq&\max_{x' \in \friend(x)} \frac{\dis^2(x,z'') - \dis^2(x,x') }{2\dis(x',z'')} - \frac{\dis(x,z'')- \Dis(x,\friend(x))}{2}\label{eq1: 1}\\
=&\frac{\dis^2(x,z'') - \dis^2(x,z') }{2\dis(z',z'')} - \frac{\dis(x,z'')- \Dis(x,\friend(x))}{2}\label{eq1: 2}\\
\leq& \frac{\dis(x,z'') + \dis(x,z') }{2} - \frac{\dis(x,z'')- \Dis(x,\friend(x))}{2}\label{eq1: 3}\\
=&\frac{\dis(x,z') + \Dis(x,\friend(x))}{2} \\
\leq &\frac{\Dis(x,\friend(x))+ \Dis(x,\nfriend(x))}{2}
\label{eq1: 5}
\end{align}
\end{small}
The first inequality (\Cref{eq1: 1}) holds due to the definition of the minimum function.
The second inequality (\Cref{eq1: 3}) is due to the triangle inequality.
The last inequality (\Cref{eq1: 5}) holds because, since $x$ is a safe input, the maximal distance between $x$ and $z'$ can't be greater than the distance from $x$ to $\nfriend(x)$.
\end{proof}

Notice that the above bound is tight, in the sense that there exists an example witnessing the exact bound, as shown in~\Cref{fig:bound} below.
\begin{figure}[h!]
    \centering
    \includegraphics[width=0.50\textwidth]{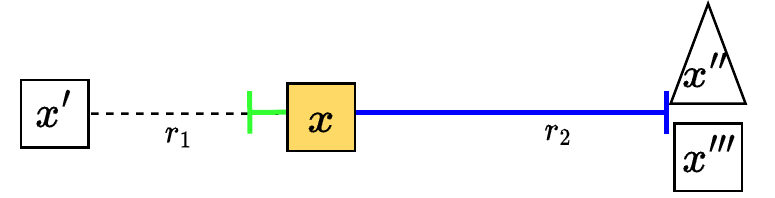}
    \caption{Example of a input $x$ with $|\sep(x)-\stab(x)|=\frac{\Dis(x,\nfriend(x)) + \Dis(x,\friend(x))}{2} $}
    \label{fig:bound}
\end{figure}

\subsection{Calibration of the Geometric Separation}
\label{sec:alg}
In this section, we use the geometric notions of $\stab(x)$ and $\sep(x)$ to derive confidence estimations ($\con(x)$). Notice that $\con(x) \in (0,1)$ while the geometric notions are in $(-\infty, +\infty)$. Next, we explain how to translate between the two. 

%
%
\label{sec:conf}
%
For each value of $\stab(x)$ ($\sep(x)$), we need to match a confidence value.
To do so,  we split the data into a Validation set, $\vs$, which is disjoint from the train and test sets. Such a methodology is commonly used in post-hoc calibration methods~\cite{CNNCalibration,platt1999probabilistic,kull2017beta,mozafari2018attended,tomani2022parameterized,pmlr-v119-zhang20k,gupta2021distribution,Ana2019Verified}.
We then measure the accuracy for inputs with similar $\stab(x)$ (or $\sep(x)$) on $\vs$. 
%
%
At this point, we have pairs $(y,z)$ where $y$ is a geometric separation value, and $z$ is the desired confidence value (as measured by the accuracy on $\vs$). The next step is to find a low-dimensionality function that maximizes accuracy.

Hence, we perform a fitting between $\stab$ (or $\sep$) values and the ratios of correct classifications (on $\vs$) for each unique value. E.g., if for $\stab$ value of $10$ we see that 90\% of the points are classified correctly, then we'll add the pair $\langle 10, 0.9 \rangle$ to the fitting function. Intuitively, we expect very low confidence values for highly negative distances and approach 100\% confidence when the distances are large and positive. 

\section{Experimental Results}
\label{sec:results}
In this section, we evaluate the effectiveness of our geometric approach. First, we explain the evaluation methodology in~\Cref{sec:method}, including the datasets and models. Then we continue our experiment results step by step by gradually explaining the tradeoffs and design decisions we take throughout this work.

\subsection{Methodology}\label{sec:method}
\subsubsection{Datasets}\label{sec:datasets}
Our evaluation uses the following standard datasets:
\begin{itemize}
   \item \emph{Modified National Institute of Standards and Technology database (MNIST)}~\cite{MNIST}.
   A dataset that consists of  hand-written images designed for training various image processing systems. It includes 70,000 28×28 grayscale images belonging to one of ten labels.
   
   \item \emph{Fashion MNIST (Fashion)}~\cite{fashion_MNIST}.
   A dataset comprising of 28×28 grayscale images of 70,000 fashion products from 10 categories. 

   \item \emph{German Traffic Signs Recognition Benchmark (GTSRB)}~\cite{GTSRB}.
   A large image set of traffic signs for
    the single-image, multi-class classification problem. It
    consists of 50,000 RGB images of traffic signs, belonging to 43 classes.

    \item \emph{American Sign Language (SignLang)}~\cite{SignLanguage}.
    A database of hand gestures representing a multi-class problem with 24 classes of letters. It consists of 30,000 28×28 grayscale images.
    
    \item \emph{Canadian Institute for Advanced Research (CIFAR10)}~\cite{CIFAR10}. A dataset containing 32x32 RGB images of 60,000 objects from 10 classes.

\end{itemize}
For each dataset, we randomly partitioned the data into  three subsets: train set $\tr$ (60\%), validation set $\vs$ (20\%) and test set $\ts$ (20\%). 
As is standard practice, we used normalized datasets (e.g.,  the same image size for all images), see ~\cite{Code} for details.
The train set is used to calculate (fast-)separation and train the model. 
The validation set is used to evaluate the confidence estimation associated with each (fast-)separation value. These values, in turn, are used to fit an isotonic function.
Finally, the test set is used to evaluate the confidence on new inputs that were \emph{not} present in the train and validation sets. 


\subsubsection{Models}\label{sec:models}
In our evaluation, we use the following popular machine learning models: Random Forest (RF) \cite{RFTheory}, Gradient Boosting Decision Trees (GB)~\cite{GBTheory}, and Convolutional Neural Network (CNN)~\cite{CNN}.
We chose these models because they are different: RF and GB are tree-based, while CNN is a neural network. 
For RF and GB, we configured the hyperparameters (e.g., the maximal depth of trees) by cross-validation on the train set via the random search technique~\cite{bergstra2012random}.
For CNN, we used the configuration suggested by practitioners. 
%
Our specific configurations as well as the accuracy scores of each of the models are detailed in~\cite{Code}. 



\subsubsection{Evaluation Algorithms}
To evaluate our method, we compare our (fast-)separation-based confidence estimation to the following methods:
the built-in isotonic regression calibration implemented by Sklearn library, ${Iso}$ \cite{Zadorozny2002Transforming}; 
the built-in Platt scaling calibration method implemented by Sklearn library, ${Platt}$~\cite{platt1999probabilistic}; 
the scaling-binning calibrator, ${SBC}$~\cite{Ana2019Verified} implemented by the same authors repository;
the histogram-binning, ${HB}$~\cite{gupta2021distribution} implemented by the same authors repository;
the beta calibrator, ${Beta}$~\cite{kull2017beta} implemented by~\cite{Kueppers_2020_CVPR_Workshops};
the bayesian binning into quantiles calibrator, ${BBQ}$~\cite{naeini2015obtaining} implemented by~\cite{Kueppers_2020_CVPR_Workshops};
the temperature scaling calibrator, ${TS}$~\cite{pmlr-v70-guo17a}  implemented by~\cite{kerrigan2021combining}; 
and the ensemble temperature scaling calibrator, ${ETS}$~\cite{pmlr-v119-zhang20k}  implemented by~\cite{kerrigan2021combining}.
%
Notice that $TS$ and $ETS$ are calibration methods for neural networks thus we only apply those to CNNs. 

Each method receives the same baseline model as an input yielding a slightly different calibrated model. 
%
Note that our method is evaluated against the uncalibrated model as our method does not affect the model. Moreover, it allows us to compare our method against different calibration methods, as shown in~\Cref{tab:mainresult}.

To evaluate the confidence predictions, we use the \emph{Expected Calibration Error (ECE)}, which is a standard method to evaluate confidence calibration of a model~\cite{Xing2020distance,Krishnan2020Improving}.
Concretely, the predictions sample of size $n$ are partitioned into $M$ equally spaced bins  $(B_{m})_{m\leq M}$, and ECE measures the difference between the sample accuracy in the $m^{th}$ bin and the  the average
confidence in it~\cite{naeini2015obtaining}. 
Formally, ECE is calculated by the following formula: 
\[E C E=\sum_{m=1}^{M} \frac{\left|B_{m}\right|}{n} \left| \operatorname{acc}\left(B_{m}\right)-\operatorname{conf} \left(B_{m}\right)\right|\] 
where:
$\operatorname{acc}\left(B_{m}\right)=\frac{1}{\left| B_{m}\right|} \cdot \left| \{x\in B_m : \cls(x)~\text{is correct} \}\right| $, and  $\operatorname{conf}\left(B_{m}\right)=\frac{1}{\left|B_{m}\right|} \sum_{x \in B_{m}} {\con(x)}$.

\color{black}%
%

\subsection{Empirical Study}
\subsubsection{Distance metrics}
\label{sec:metrics}

%
%



As mentioned in~\Cref{sec:sep}, the notion of geometric separation is applicable to any norm. In fact, as shown in~\Cref{prop:stab_all_l}, the fast-separation approximation provides a zone under any norm. 
Thus, we have evaluated the ECE obtained from fast-separation  under different norms. The results are given in~\Cref{tab:mainresult}. As can be observed, the ECE is low regardless of the selection of norm indicating the attractiveness of the geometric signal.  However, while some norms are more accurate for some datasets, there is no universally superior norm. Thus, the following experiments focus on the  $L_2$ norm from the reasons specified in~\cref{sec:sep}. 

\renewcommand{\arraystretch}{1.3}
\begin{table*}[t!]
\centering
\scalebox{0.7}{
\large
\begin{tabular}{cccccc}
\hline
\rowcolor{Gainsboro!90} \textbf{Dataset}          & \textbf{Model} & \textbf{$L_1$} & \textbf{$L_2$} & \textbf{$L_{\infty}$} \\ \hline
                          & CNN            & 0.17 {\tiny\textpm0.03}      & 0.18 {\tiny\textpm0.05}     & 0.08 {\tiny\textpm0.03}       \\ \rowcolor{Gainsboro!40}
                          & GB             & 0.28 {\tiny\textpm0.09}      & 0.36 {\tiny\textpm0.07}     & 0.37 {\tiny\textpm0.06}       \\ \multirow{-3}{*}{\rotatebox[origin=c]{90}{MNIST}}   
						 & RF             & 0.37 {\tiny\textpm0.07}      & 0.39 {\tiny\textpm0.06}     & 0.37 {\tiny\textpm0.05}       \\ \hline    
                          & CNN            & 0.38 {\tiny\textpm0.14}      & 0.42 {\tiny\textpm0.15}     & 0.36 {\tiny\textpm0.16}       \\ \rowcolor{Gainsboro!40}
                          & GB             & 1.08 {\tiny\textpm0.19}      & 0.65 {\tiny\textpm0.11}     & 0.41 {\tiny\textpm0.09}       \\ \multirow{-3}{*}{\rotatebox[origin=c]{90}{GTSRB}}
						 & RF             & 0.54 {\tiny\textpm0.19}      & 0.37 {\tiny\textpm0.04}     & 0.32 {\tiny\textpm0.05}       \\ \hline 
                          & CNN            & 0.05 {\tiny\textpm0.03}      & 0.05 {\tiny\textpm0.04}     & 0.05 {\tiny\textpm0.03}       \\ \rowcolor{Gainsboro!40}
                          & GB             & 0.00 {\tiny\textpm0.00}      & 0.08 {\tiny\textpm0.03}     & 0.17 {\tiny\textpm0.02}       \\ \multirow{-3}{*}{\rotatebox[origin=c]{90}{SignLang}}
						 & RF             & 0.00 {\tiny\textpm0.00}      & 0.08 {\tiny\textpm0.02}     & 0.14 {\tiny\textpm0.03}       \\ \hline 
                          & CNN            & 0.74 {\tiny\textpm0.07}      & 0.79 {\tiny\textpm0.15}     & 0.55 {\tiny\textpm0.12}       \\ \rowcolor{Gainsboro!40}
                          & GB             & 0.64 {\tiny\textpm0.21}      & 0.73 {\tiny\textpm0.13}     & 0.85 {\tiny\textpm0.17}       \\ \multirow{-3}{*}{\rotatebox[origin=c]{90}{Fashion}} 
						 & RF             & 0.68 {\tiny\textpm0.13}      & 0.74 {\tiny\textpm0.16}     & 0.83 {\tiny\textpm0.14}       \\ \hline 
                          & CNN            & 1.20 {\tiny\textpm0.19}      & 1.20 {\tiny\textpm0.30}     & 3.03 {\tiny\textpm0.79}       \\ \rowcolor{Gainsboro!40}
                          & GB             & 1.59 {\tiny\textpm0.51}      & 1.25 {\tiny\textpm0.21}     & 1.17 {\tiny\textpm0.24}       \\ \multirow{-3}{*}{\rotatebox[origin=c]{90}{CIFAR10}}
						 & RF             & 1.08 {\tiny\textpm0.45}      & 1.15 {\tiny\textpm0.24}     & 1.30 {\tiny\textpm0.18}       \\ \hline 
\end{tabular}
}
\caption{ECE(\%) measures with $95\%$ confidence intervals comparing the results of the fast-separation-based method using $L_1,L_2$ and $L_{\infty}$ norms.}
\label{tab:norms}
\end{table*}
\subsubsection{Fitting Function}

%
As mentioned in~\Cref{sec:alg}, for our fitting function we can  use any existing  calibration function.
Post-hoc calibration methods based on fitting functions typically use either a logistic (Sigmoid) or an isotonic regression~\cite{Zadorozny2002Transforming}.
Isotonic regression fits a non-decreasing free-form line to a sequence of observations. In comparison, Sigmoid is a continuous step function.  
We used both fitting functions on our fast-separation values and obtained similar accuracy. 
We opt here to present the isotonic regression as it provides the best empirical results, as motivated by~\Cref{fig:Sigmoid}.

\Cref{fig:Sigmoid} illustrates an example of the success ratio of the Random Forest model for MNIST inputs with varying values of $\stab$ scores (similar behavior was observed for the various models and datasets).  We clustered inputs with a similar score together (into 50 bins overall) as each classification is correct or not, and we are looking for the average.  The black line represents the Sigmoid function, and the green line represents the isotonic regression. 
As can be observed, both regressions are nearly identical on all the points with positive $\stab$ values. We eventually chose isotonic regression because it better fitted the few points with negative $\stab$ values. Interestingly, these points were consistently a poor fit for the Sigmoid regression rendering it slightly less accurate on average. 
Also, observe that the transition is around the value 0, indicating that the distinction between safe and dangerous points is meaningful in confidence evaluation.
 \begin{figure}[t!]
     \centering
     \includegraphics[width=0.7\textwidth]{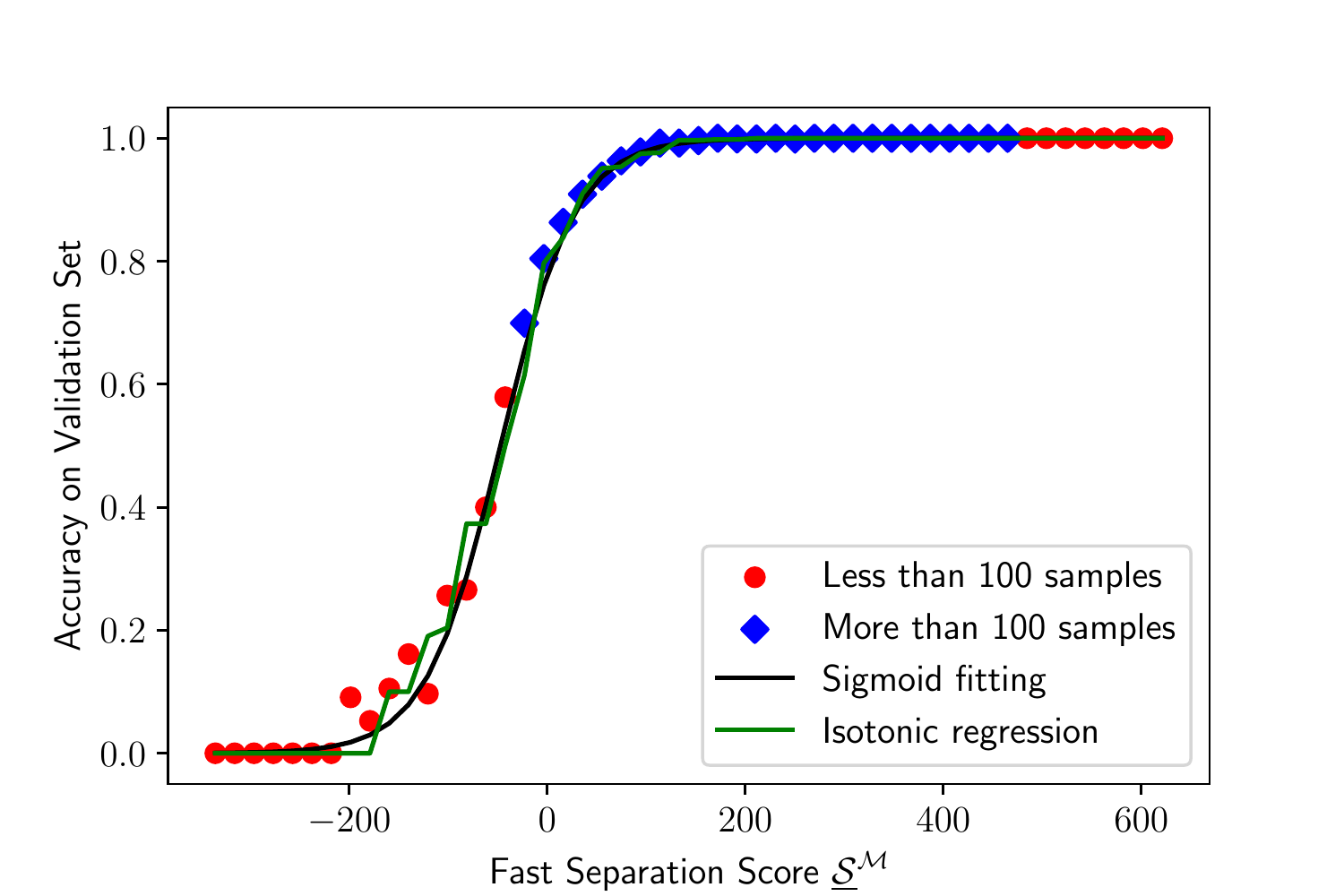}
     \caption{An illustration of the inputs to the fitting function (blue diamonds and red dots), and the functions fitted by Sigmoid (black line) and isotonic regression (green line). The inputs are for the MNIST dataset and the Random Forest model.
      }
     \label{fig:Sigmoid}
 \end{figure}

%
%

\subsection{Confidence Evaluation}
\setlength{\tabcolsep}{5.5pt}
\renewcommand{\arraystretch}{1.5}
\begin{table*}[t!]
\centering
\caption{ECE(\%) measures with $95\%$ confidence intervals when varying the calibration method, model, and dataset.}
\resizebox{\textwidth}{!}{

\LARGE
\begin{tabular}{cccccccccccccc}
\hline

\rowcolor{Gainsboro!90}
\LARGE{\textbf{Dataset}}   & 
\LARGE{\textbf{Model}} & 
\LARGE{\textbf{$\mathbf{\stab}$}}  & 
\LARGE{\textbf{$\mathbf{\sep}$}}   & 
\LARGE{\textbf{Iso}}  & 
\LARGE{\textbf{Platt}}  & 
\LARGE{\textbf{SBC}}  & 
\LARGE{\textbf{HB}}  &
\LARGE{\textbf{BBQ}} &
\LARGE{\textbf{Beta}} &
\LARGE{\textbf{TS}}  & 
\LARGE{\textbf{ETS}} 
\\ \hline
                                      & CNN    &         \textbf{0.15\small{±.01}}  &        \textbf{0.15\small{±0.01}} &            0.17\small{±0.01} &        0.52\small{±0.04} &          8.91\small{±0.16} &       0.32\small{±0.02}    &   0.22\small{±0.01}  &  0.64\small{±0.02}  &     0.20\small{±0.01}   &    0.20\small{±0.01} \\         
\rowcolor{Gainsboro!40}               & RF     &         \textbf{0.35\small{±.02}}  &                 0.36\small{±0.02} &            0.92\small{±0.03} &        1.49\small{±0.02} &          3.92\small{±0.11} &       0.46\small{±0.02}    &   1.13\small{±0.03}  &  0.37\small{±0.02}  &             -           &     -               \\        
                                      & GB     &         \textbf{0.34\small{±.02}}  &        \textbf{0.34\small{±0.02}} &            1.74\small{±0.03} &        1.97\small{±0.03} &          8.46\small{±0.07} &       0.45\small{±0.02}    &   0.65\small{±0.03}  &  0.47\small{±0.02}  &             -           &     -               \\  \hline \multirow{-5}{*}{\rotatebox[origin=c]{90}{MNIST}}        
                                      & CNN    &         \textbf{0.37\small{±.04}}  &        \textbf{0.37\small{±0.04}} &            0.38\small{±0.04} &        2.83\small{±0.53} &         29.01\small{±0.49} &       1.22\small{±0.18}    &   1.08\small{±0.21}  &  1.98\small{±0.25}  &     0.90\small{±0.11}   &    0.77\small{±0.09} \\           
\rowcolor{Gainsboro!40}               & RF     &         \textbf{0.37\small{±.02}}  &                 0.38\small{±0.02} &            2.55\small{±0.04} &        4.19\small{±0.03} &         13.99\small{±0.11} &       0.85\small{±0.05}    &   3.08\small{±0.04}  &  0.56\small{±0.03}  &             -           &     -               \\             
                                      & GB     &         \textbf{0.61\small{±.03}}  &                 0.63\small{±0.03} &           10.04\small{±0.07} &       19.63\small{±0.83} &         31.25\small{±0.12} &       1.42\small{±0.05}    &   9.28\small{±0.11}  &  5.36\small{±0.10}  &             -           &     -               \\ \hline  \multirow{-5}{*}{\rotatebox[origin=c]{90}{GTSRB}}      
                                      & CNN    &         \textbf{0.09\small{±.05}}  &                 0.10\small{±0.06} &   \textbf{0.09\small{±0.05}} &        0.12\small{±0.07} &         17.77\small{±0.21} &       1.24\small{±1.03}    &   1.24\small{±1.03}  &  1.24\small{±1.04}  &     0.11\small{±0.01}   &    0.12\small{±0.01} \\        
\rowcolor{Gainsboro!40}               & RF     &         \textbf{0.08\small{±.01}}  &        \textbf{0.08\small{±0.01}} &            0.46\small{±0.02} &        1.76\small{±0.02} &         17.34\small{±0.18} &       0.16\small{±0.02}    &   0.86\small{±0.02}  &  0.29\small{±0.01}  &             -           &     -               \\        
                                      & GB     &         \textbf{0.07\small{±.01}}  &        \textbf{0.07\small{±0.01}} &            4.01\small{±0.06} &        5.93\small{±0.06} &         31.01\small{±0.08} &       0.46\small{±0.03}    &   0.78\small{±0.05}  &  0.70\small{±0.03}  &             -           &     -               \\ \hline   \multirow{-5}{*}{\rotatebox[origin=c]{90}{SignLang}}     
                                      & CNN    &                  0.75\small{±.03}  &                 0.75\small{±0.04} &   \textbf{0.71\small{±0.03}} &        6.60\small{±0.72} &          7.36\small{±0.20} &       1.10\small{±0.05}    &   2.18\small{±0.15}  &  9.15\small{±0.10}  &     0.82\small{±0.04}   &    0.89\small{±0.04} \\        
\rowcolor{Gainsboro!40}               & RF     &         \textbf{0.78\small{±.04}}  &                 0.82\small{±0.04} &            1.03\small{±0.05} &      3.75\small{±0.04}	&          3.52\small{±0.10} &       1.07\small{±0.05}    &   1.23\small{±0.05}  &  0.83\small{±0.03}  &             -           &     -               \\             
                                      & GB     &         \textbf{0.79\small{±.04}}  &        \textbf{0.79\small{±0.04}} &            3.82\small{±0.06} &        5.01\small{±0.65}	&          3.90\small{±0.12} &       1.01\small{±0.04}    &   1.41\small{±0.05}  &  0.97\small{±0.05}  &             -           &     -               \\ \hline    \multirow{-5}{*}{\rotatebox[origin=c]{90}{Fashion}}  
                                      & CNN    &         \textbf{1.12\small{±.07}}  &                 1.16\small{±0.07} &            1.28\small{±0.06} &        6.05\small{±0.21}	&          3.57\small{±0.10} &       4.10\small{±0.10}    &   5.31\small{±0.24}  & 24.76\small{±0.21}  &     3.68\small{±0.08}   &    3.45\small{±0.12} \\        
\rowcolor{Gainsboro!40}               & RF     &         \textbf{1.37\small{±.07}}  &                 1.38\small{±0.07} &            3.33\small{±0.07} &        4.54\small{±0.09} &          3.01\small{±0.09} &       2.27\small{±0.09}    &   3.84\small{±0.08}  &  1.65\small{±0.06}  &             -           &     -               \\           
                                      & GB     &                  1.41\small{±.07}  &        \textbf{1.38\small{±0.06}} &           7.70\small{±0.08} &        8.58\small{±0.09}	&          2.63\small{±0.09} &       2.40\small{±0.10}    &   1.51\small{±0.07}  &  2.94\small{±0.08}  &             -           &     -               \\ \hline   \multirow{-5}{*}{\rotatebox[origin=c]{90}{CIFAR10}}    
\end{tabular}}
\label{tab:mainresult}
\end{table*}

\begin{table*}[t!]
\centering
\caption{Relative improvement percentage of ECE of ${\stab}$ over other calibration methods.}
\footnotesize
\begin{tabular}{cccccccccccc}
\hline

\rowcolor{Gainsboro!90}
\textbf{Dataset} & 
\textbf{Model} & 
\textbf{Iso} & 
\textbf{Platt}  & 
\textbf{SBC}  & 
\textbf{HB}  &
\textbf{BBQ} &
\textbf{Beta} &
\textbf{TS}  & 
\textbf{ETS} 
\\ \hline
                                      & CNN   &       11.8\% &    71.2\% &    98.3\% &        53.1\% &  31.8\% &  76.6\% &        25.0\% &    25.0\%  \\         
\rowcolor{Gainsboro!40}               & RF    &       62.0\% &    76.5\% &    91.1\% &        23.9\% &  69.0\% &   5.4\% &             - &       -    \\        
                                      & GB    &       80.5\% &    82.7\% &    96.0\% &        24.4\% &  47.7\% &  27.7\% &             - &       -    \\  \hline \multirow{-5}{*}{\rotatebox[origin=c]{90}{MNIST}}        
                                      & CNN   &        2.6\% &    86.9\% &    98.7\% &        69.7\% &  65.7\% &  81.3\% &        58.9\% &    51.9\%  \\           
\rowcolor{Gainsboro!40}               & RF    &       85.5\% &    91.2\% &    97.4\% &        56.5\% &  88.0\% &  33.9\% &             - &       -    \\             
                                      & GB    &       93.9\% &    96.9\% &    98.0\% &        57.0\% &  93.4\% &  88.6\% &             - &       -    \\ \hline  \multirow{-5}{*}{\rotatebox[origin=c]{90}{GTSRB}}      
                                      & CNN   &        0.0\% &    25.0\% &    99.5\% &        92.7\% &  92.7\% &  92.7\% &        18.2\% &    25.0\%  \\        
\rowcolor{Gainsboro!40}               & RF    &       82.6\% &    95.5\% &    99.5\% &        50.0\% &  90.7\% &  72.4\% &             - &       -    \\        
                                      & GB    &       98.3\% &    98.8\% &    99.8\% &        84.8\% &  91.0\% &  90.0\% &             - &       -    \\ \hline   \multirow{-5}{*}{\rotatebox[origin=c]{90}{SignLang}}     
                                      & CNN   &       -5.6\% &    88.6\% &    89.8\% &        31.8\% &  65.6\% &  91.8\% &         8.5\% &    15.7\%  \\        
\rowcolor{Gainsboro!40}               & RF    &       24.3\% &    79.2\% &    77.8\% &        27.1\% &  36.6\% &   6.0\% &             - &       -    \\             
                                      & GB    &       79.3\% &    84.2\% &    79.7\% &        21.8\% &  44.0\% &  18.6\% &             - &       -    \\ \hline    \multirow{-5}{*}{\rotatebox[origin=c]{90}{Fashion}}  
                                      & CNN   &       12.5\% &    81.5\% &    68.6\% &        72.7\% &  78.9\% &  95.5\% &        69.6\% &    67.5\%  \\        
\rowcolor{Gainsboro!40}               & RF    &       58.9\% &    69.8\% &    54.5\% &        39.6\% &  64.3\% &  17.0\% &             - &       -    \\           
                                      & GB    &       81.7\% &    83.6\% &    46.4\% &        41.2\% &   6.6\% &  52.0\% &             - &       -    \\ \hline   \multirow{-5}{*}{\rotatebox[origin=c]{90}{CIFAR10}}    
\end{tabular}
\label{tab:relative_result}
\end{table*}

This section presents the experimental results of the confidence estimation.  

\subsubsection{Estimating Confidence}

\cref{tab:mainresult} presents the main experimental results of our work. The table summarizes  ECEs for our method (with bin size $15$). 
Each entry in the table describes the ECE and the 95\% confidence interval. 
%
We highlight the most accurate method for each experiment in bold. 
In this experiment, we perform one hundred random splits of the data into train, validation, and test sets for each model and dataset. We then measure the ECE of the confidence estimation for all test set inputs, average the result and take the 95\% confidence intervals. 
\footnote{For $Plat$ and $Iso$ we used the standard SKlearn implementation. However, we used Pytorch for CNN models, and Pytorch does not have $Plat$ and $Iso$, so we implemented them for our Pytorch-based CNN models.}
First, observe that $\stab$ and $\sep$ yield very similar ECEs, and that the differences between them are usually statistically insignificant. Thus, we conclude that $\stab$ is a good approximation of $\sep$ despite being considerably  simpler to compute. 
The next interesting comparison is between $\stab$ and $Iso$. We use the same fitting function (isotonic regression) in both cases, but $Iso$ performs the calibration on the model's natural uncertainty estimation, and $\stab$ performs the calibration on geometric distances.
Our $\stab$ almost consistently improves the confidence estimations across the board compared to $Iso$, $Platt$, $SBC$, $HB$, $TS$, $ETS$, $Beta$, and $BBQ$.  Specifically, we derive improvements up to 99\% in almost all tested models and datasets. 
Such results demonstrate the potential of geometric signals to improve the effectiveness of uncertainty estimation. 



\cref{tab:relative_result} describes the improvement of our fast-separation-based method over recently proposed posthoc calibration techniques. The improvement is calculated using the ratio of the difference between our ECE and the competitor's ECE. Observe that our method always improves the alternatives except for CNNs on the Fashion dataset, where it loses by 5\%. Such results position our geometric method as a competitive approach for confidence estimation. However, note that our fast-separtion can be used alongside the existing methods.

\subsubsection{Tabular Data}
Fast-separation was designed for 
image data sets since they appear to be most governed by geometry, that is, different images will likely be geometrically separable.
%
Nonetheless, as a controlled experiment, we also tested our method on non-visual tabular data. Here, we have no apriori intuition that the geometric signal is feasible.  We used two datasets: Red wine quality~\cite{WINE}, which contains a total of twelve variables and 1,599 observations and six classes, and   airline passenger satisfaction~\cite{AIRLINE}, which  contains a total of twenty-five variables, 129,880 observations, and two classes. 

In most experiments, we saw a small improvement of ranging between 1\% to 77\% in accuracy. Thus, we conclude that our method achieves good results on tabular data as well. However, the improvement was not uniform and there were a few cases where Iso was superior to our own. Thus,  the geometric signal may also be useful for non-visual data but further investigations are required to adapt the method to various datasets.

\section{Optimizing Performance}
\label{sec:optimization}

As shown in the previous section, the fast-separation approximation yields competitive confidence estimations promptly for small and medium-sized datasets.
Nonetheless, our approach may still be too slow to handle large datasets due to the need to calculate geometric notions on the entire training set.  To address this bottleneck, we explore the impact of several standard methods for dimensionality reduction on the quality of our approach for confidence estimation.

\subsection{Handling Large Datasets}
\label{sec:opt}
Large datasets are datasets with a large number of images or with large images with many pixels. In such cases, each calculation of fast separation requires potentially going over many comparisons that slow down the process. 
Here, we explore ways to either reduce the image size, or to reduce the number of images.\footnote{One can also directly manipulate the searching algorithm to improve the calculation complexity of the nearest neighbor by, e.g.,  randomization or special data structures. We leave exploring this option for future work.}
%
The following list reviews various known techniques for reducing the dimensionality of the data, the first four reduce the number of pixels, and the last two reduce the number of images in the set used to calculate geometric distances. 

In order to have a fair comparison, we define the reduction parameter $t$ to indicate the amount of data reduced in each method. In each method, a reduction parameter $t$ implies that we reduce the dataset size by a factor of $t^2$. E.g., in the Pooling technique, we can reduce 2x2 images into a pixel reducing the image size, while K-means would reduce the number of images, and both would reduce it by a factor of four so that the total number of pixels in the set is the same for each reduction parameter value for all the methods. 
%
%




\textbf{Pooling}~\cite{pool}  is an operation that calculates a function for patches of a feature map and uses it to create a down-sampled (pooled) feature map. For example, if one wants a 2-pool of an image, one reduces its size by 2x2, and every square of 2x2 is then represented as the output of the function on the squared elements.  Some broadly used functions for pooling are average ($pool$) and maximum ($maxpool$).

%

\textbf{Principal Component Analysis ($PCA$)}~\cite{pca} linearly transforms the data into a new coordinate system where most of the variation in the data can be described with fewer dimensions than the initial data.
For reduction parameter 2 we reduce each image to a new smaller image with a reduction factor of four in the number of pixels. 

\textbf{Resizing using a Bilinear Interpolation ($RBI$)}~\cite{bilinear} is a generalization of single dimension linear interpolation. RBI performs linear interpolation in one direction and then again in the other direction. Resizing using a Bilinear Interpolation is common in computer vision applications that are based on convolutional neural networks.
For a reduction parameter 2 we resize the image to a new image in which both the length and width are two times smaller, ending with an image four times smaller than the original one. 

\textbf{Sampling random pixels  ($Rand_{pix}$)} reduces the number of pixels in the metadata by a random sample. Notice that this approach can be viewed as the baseline for other pixel-reducing techniques. We chose the number of pixels sampled to be the original pixel number divided by the squared reduction parameter.

\textbf{K-means}~\cite{kmeans} clustering is a vector quantization method aiming to partition $n$ observations into $k$ clusters in which each observation belongs to the cluster with the nearest mean (cluster centroid). K-means clustering minimizes variances in the clusters (squared Euclidean distances). Here, we set $k$ to be the reduced dimension of the compressed dataset. E.g., if the original dataset had 10,000 images, and we set $k=1,000$, we get a reduction factor of x10 from $10,000$ dimensions to $1,000$. 
When using this method we first find the centroids of the dataset, and then use these as the metadata for calculating geometric separation.


\textbf{Sampling the training set ($Rand_{set}$)}
reduces the number of inputs in the training set, by picking a random sample. We chose the sample size to be the dataset size divided by the squared reduction parameter.

\subsection{Experimental Results}

\begin{figure}[t!]
 \centering
    \begin{tabular}{ccc}
    \subfloat[MNIST]{\includegraphics[width=0.31\linewidth]{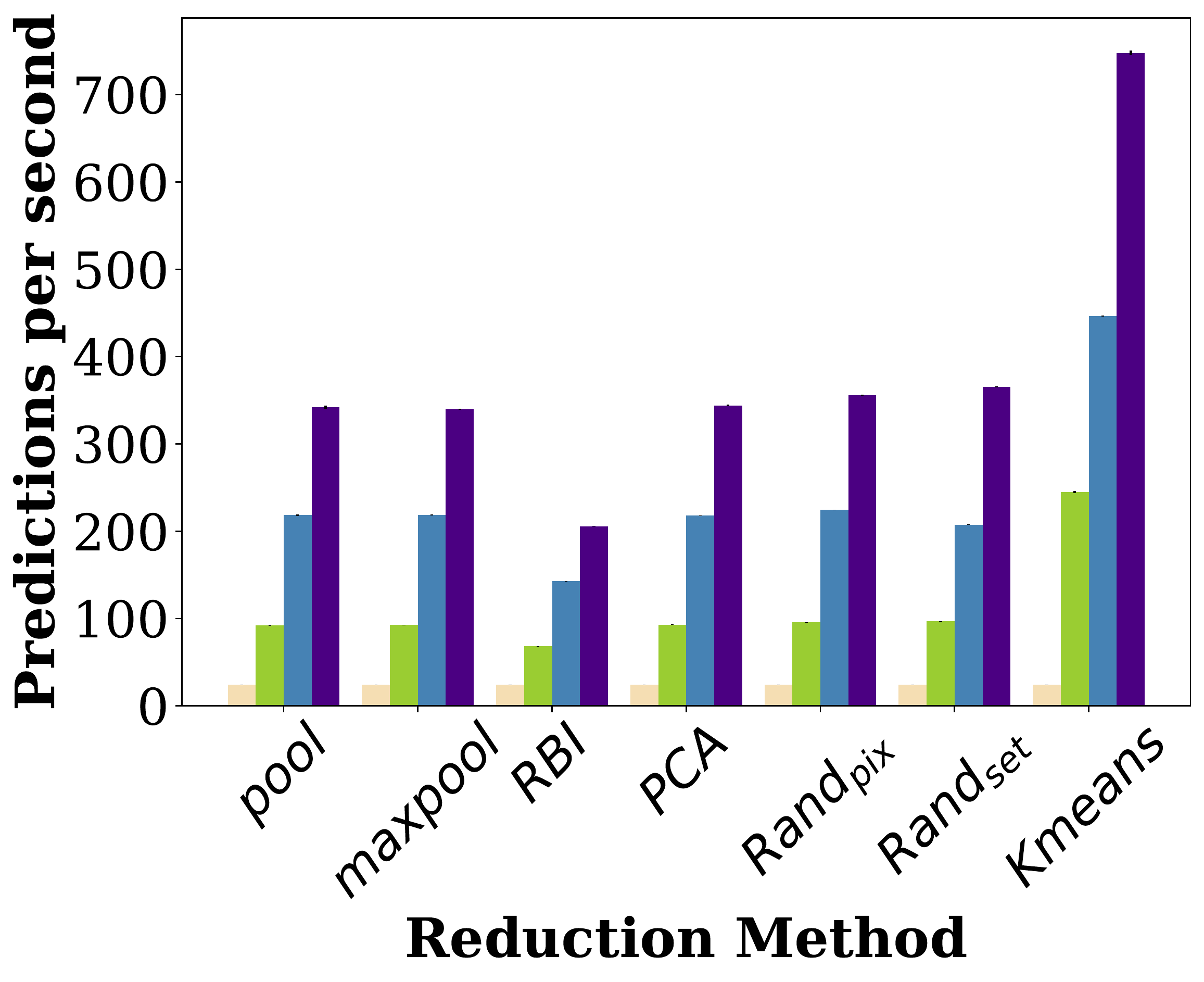}} &
    \subfloat[Fashion]{\includegraphics[width=0.31\linewidth]{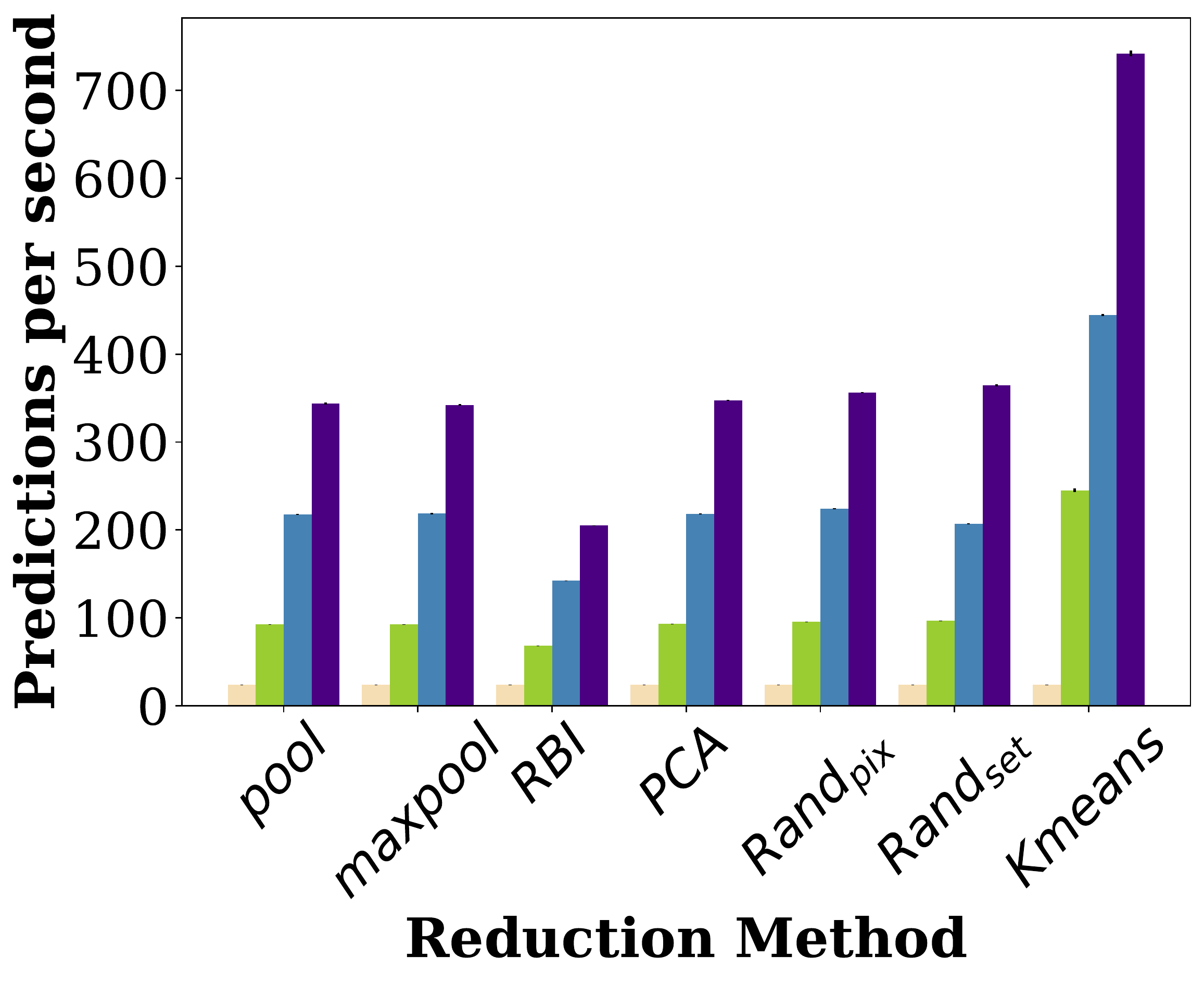} } &
  \subfloat[GTSRB]{\includegraphics[width=0.31\linewidth]{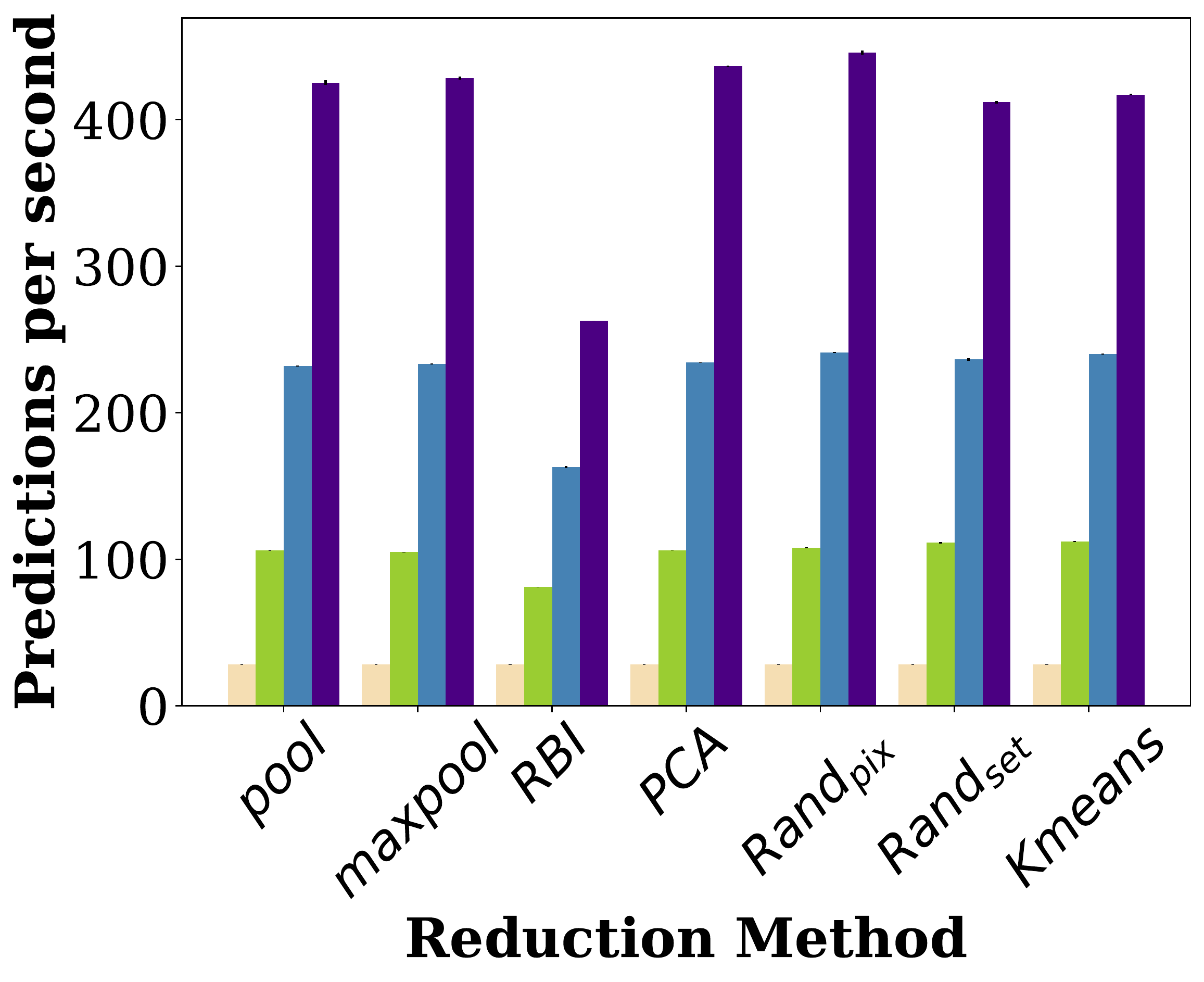}} \\  
  \includegraphics[width=0.1\linewidth]
{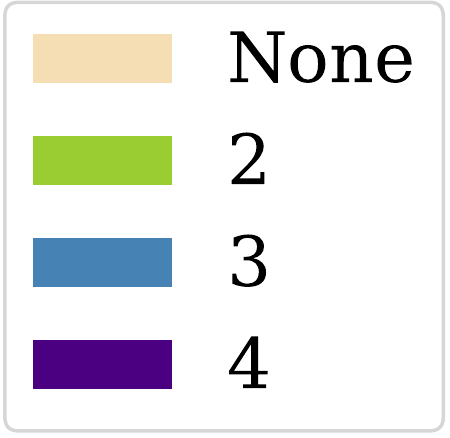}&
    \subfloat[CIFAR10]{\includegraphics[width=0.31\linewidth]{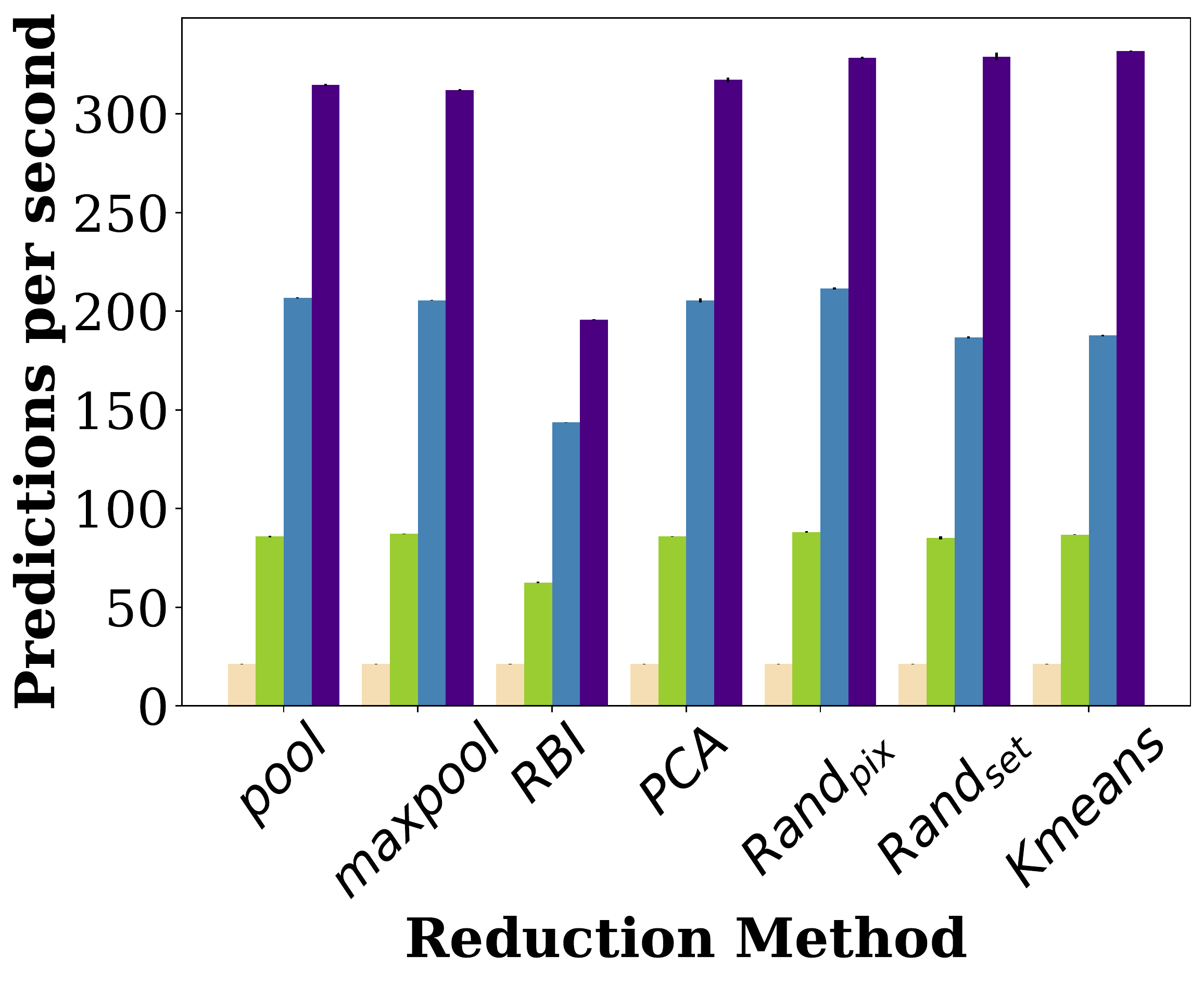} } &
  
  \subfloat[SignLanguage]{\includegraphics[width=0.31\textwidth]{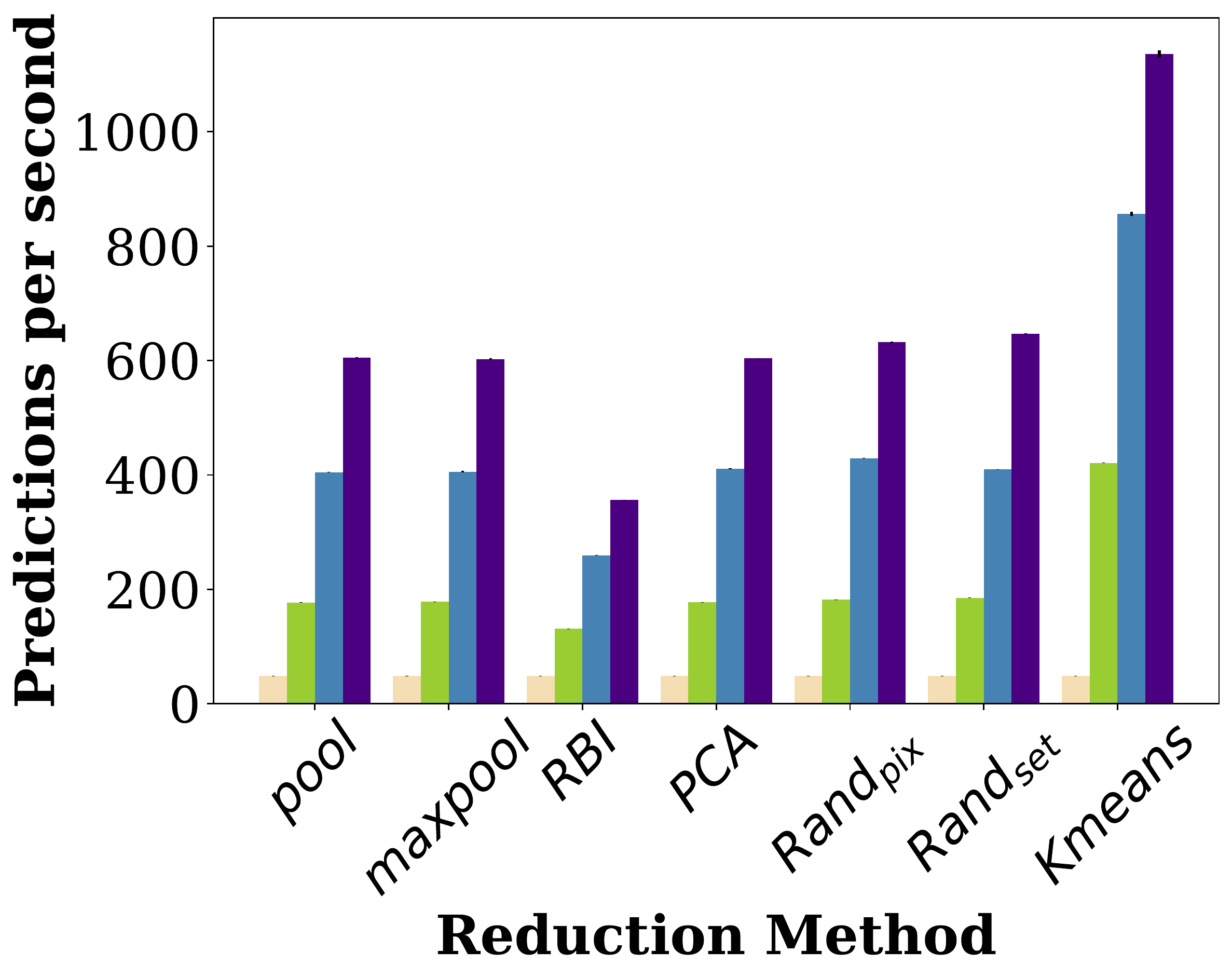} } 
    \end{tabular}
    \caption{Time comparison between various reduction methods on all datasets with Random Forest model. The error bar show the 95\% confidence interval over 10 shuffles.   The colors denote the various reduction parameters.}
    \label{fig:Time-allmethods}
\end{figure}

To evaluate the effect of these reductions on our algorithm, we apply the reduction to the whole dataset and then calculate the fast-separation values on the reduced dataset. Note that the models are trained on the original dataset, so accuracy is not affected.
For RGB images, we further changed the color to grayscale images, which reduced the size of images by a factor of 3 while keeping the image as close as possible to the original one.
The experiments were executed on a desktop PC with Intel(R) 16 Cores(TM) i7-10700 CPU @ 2.90GHz, and 16GB RAM.

\cref{fig:Time-allmethods,fig:ECE-allmethods} show a comparison of the speed and accuracy of the various methods in the Random Forest model.
As shown in~\Cref{fig:Time-allmethods},  all data optimizations increase the number of predictions per second, and we can readily reach several hundred estimations per second which is a sufficient speedup for our needs.
All methods show almost the same speedup on the algorithm for each hyperparameter value,  except for k-means which sometimes has a better speedup and RBI which has a slightly lower speedup. 
Since there is little variability in the experiments, the confidence intervals are barely visible.


%


Our experiments show that the time performance is not affected by the model. 
Thus,~\Cref{fig:Time-allmethods} presents only the results for the Random forest model, i.e., the average number of predictions per second. Observe that the number of predictions per seconds is the same for all other models. 

%
Importantly, this improvement in runtime does not come at a meaningful cost for the confidence estimation, as shown in~\Cref{fig:ECE-allmethods}.
While the error in the calibration estimation  slightly changes across different methods and reduction parameters, the changes seem insignificant.
Specifically, in the SignLanguage dataset we observe an increase in the ECE, which we believe is due to the fact that the original dataset is already quite small, rendering the datasize optimization pointless. 
Our experiments also show similar behavior across different models. For example,~\Cref{fig:ECE-Maxpool} shows that all three models obtain similar errors with maxpooling with different parameters. 
Moreover, our results outperform most state-of-the-art algorithms even with a 4-pool.

\begin{figure}[t!]
 \centering
    \begin{tabular}{ccc}
    \subfloat[MNIST]{\includegraphics[width=0.31\linewidth]{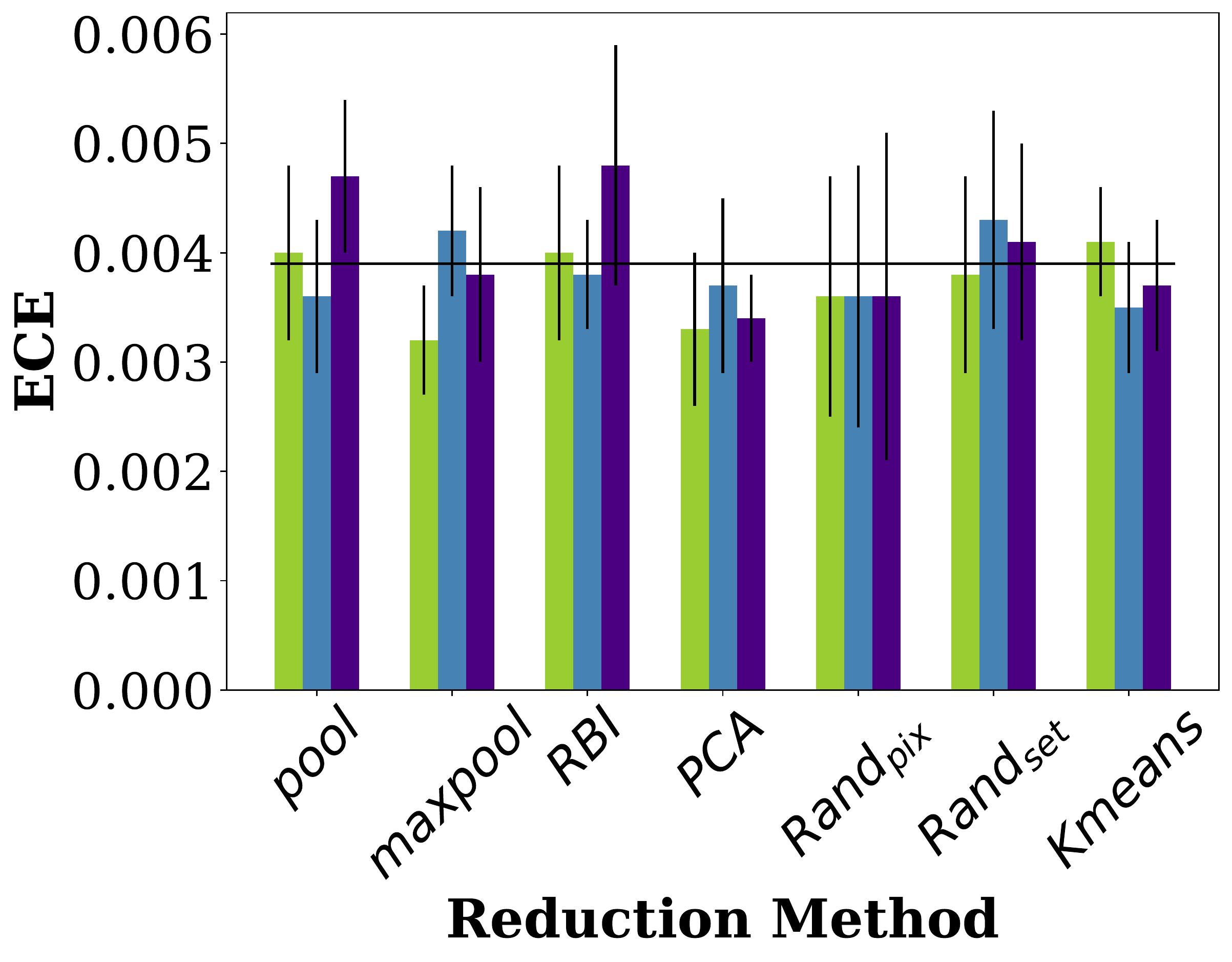}} &
    \subfloat[Fashion]{\includegraphics[width=0.31\linewidth]{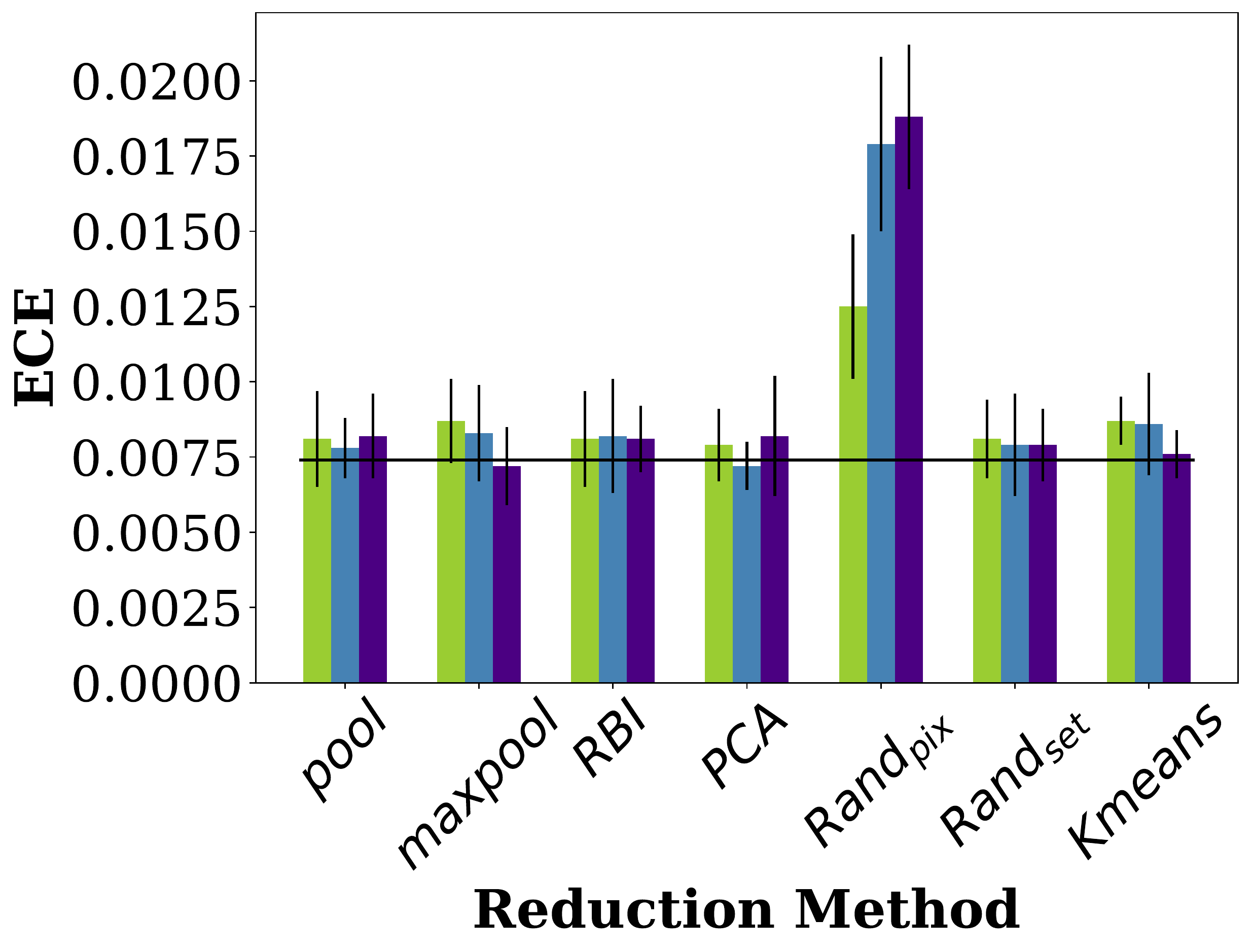}}  &
   \subfloat[GTSRB]{\includegraphics[width=0.31\linewidth]{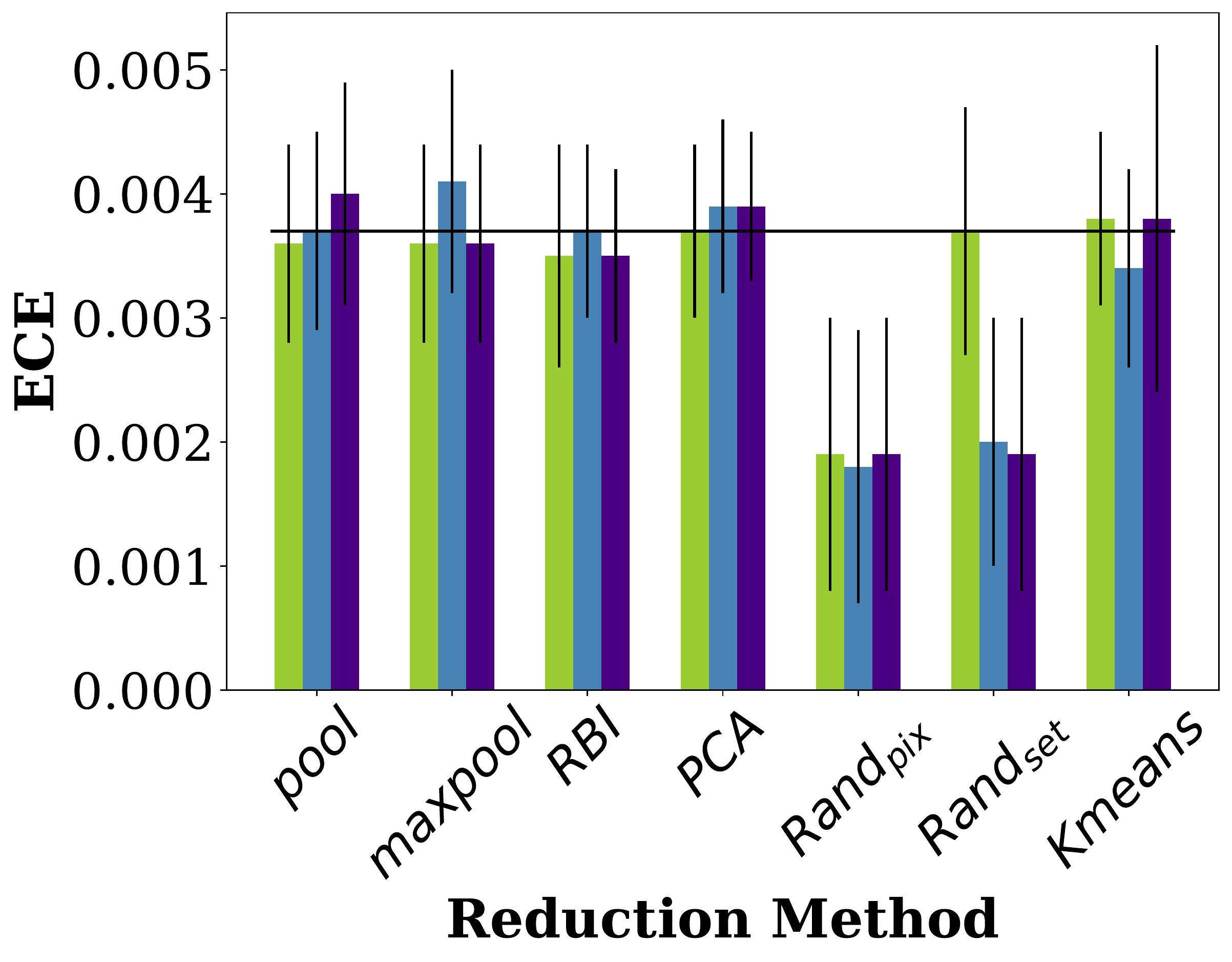}} \\
   
    \includegraphics[width=0.1\linewidth]{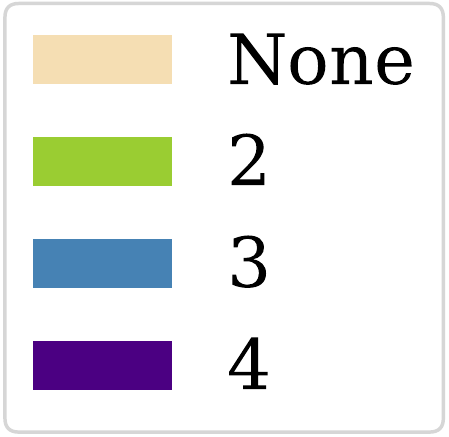} &
   \subfloat[CIFAR10]{\includegraphics[width=0.31\linewidth]{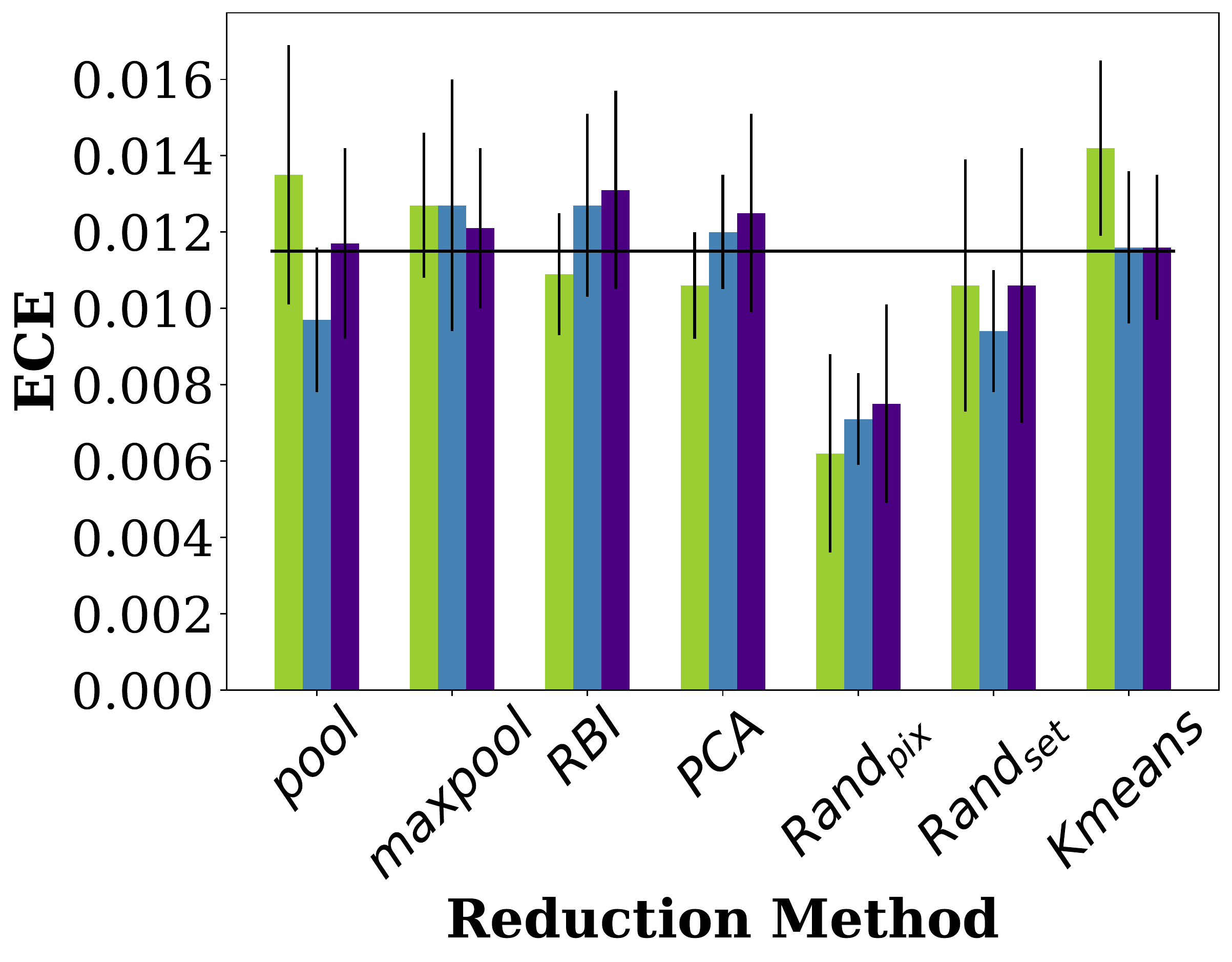}}  &

   \subfloat[SignLanguage]{\includegraphics[width=0.31\linewidth]{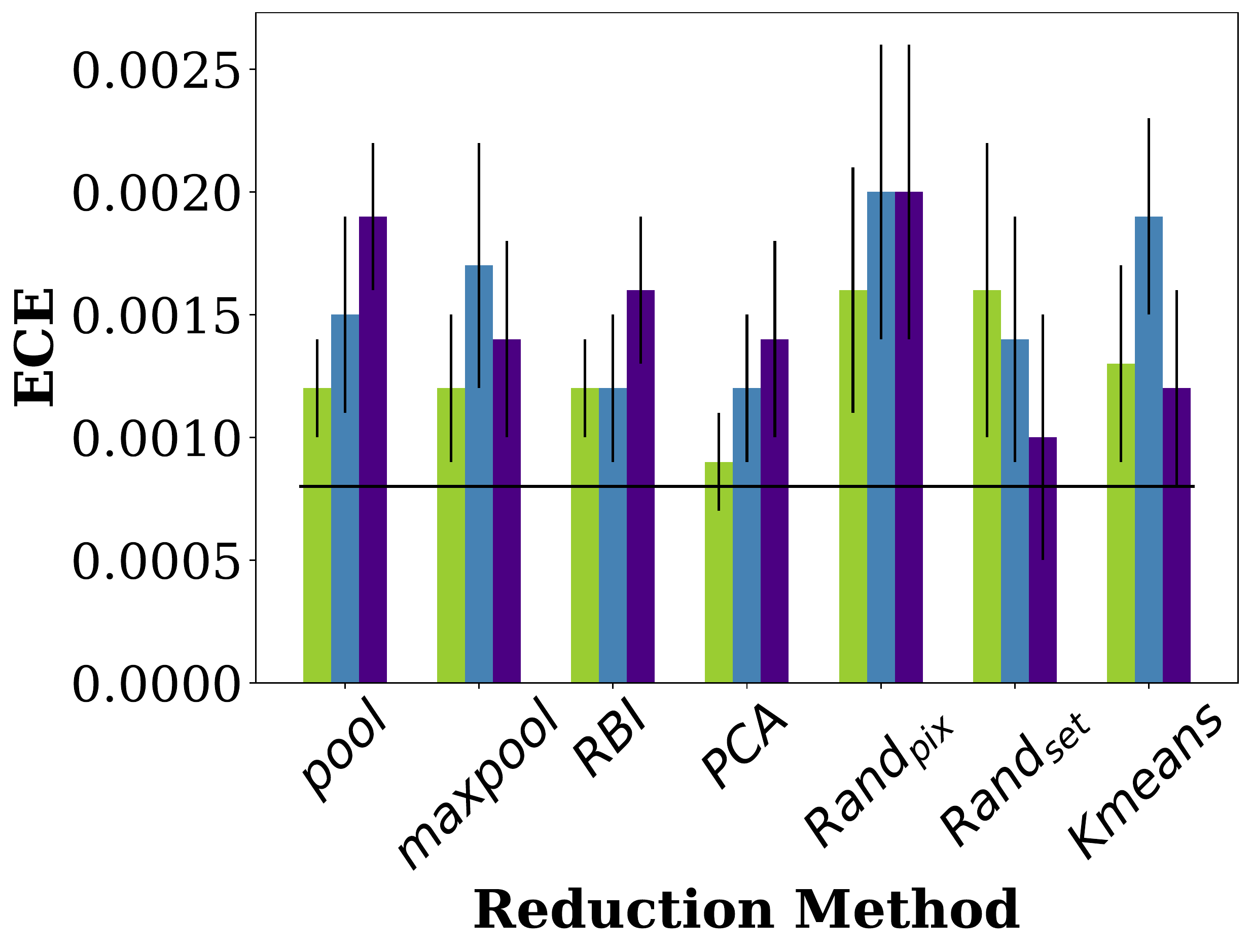}} 
    \end{tabular}
    \caption{ECE comparison between various reduction methods on all datasets with Random forest model. The error bar show the 95\% confidence interval over 10 shuffles.
    The black line shows the ECE score of the method without any reduction. The colors denote the various reduction parameters.}
    \label{fig:ECE-allmethods}
\end{figure}




\begin{figure}[t!]
 \centering
    \begin{tabular}{ccc}
    \subfloat[GB]{\includegraphics[width=0.31\linewidth]{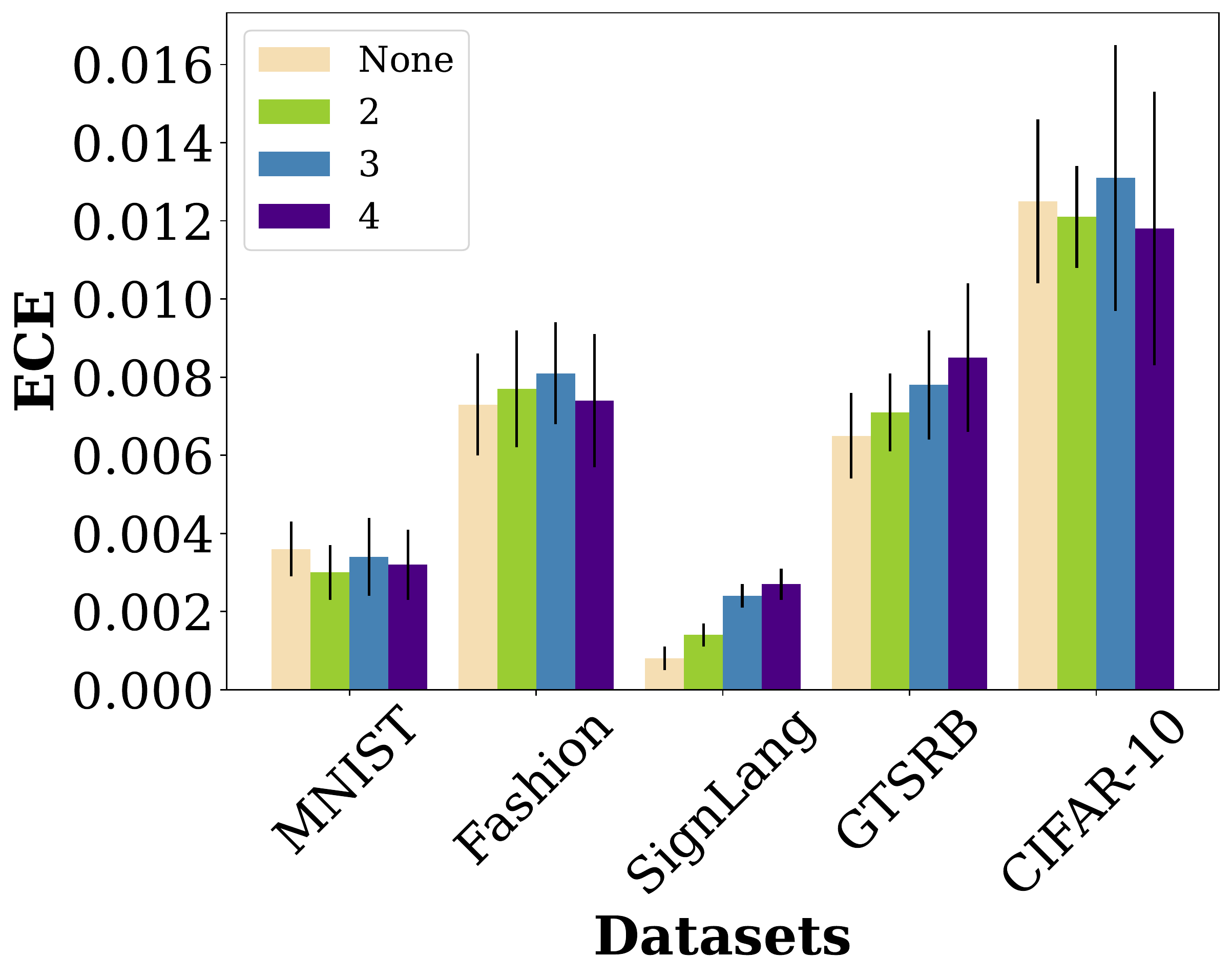}} &
    \subfloat[RF]{\includegraphics[width=0.31\linewidth]{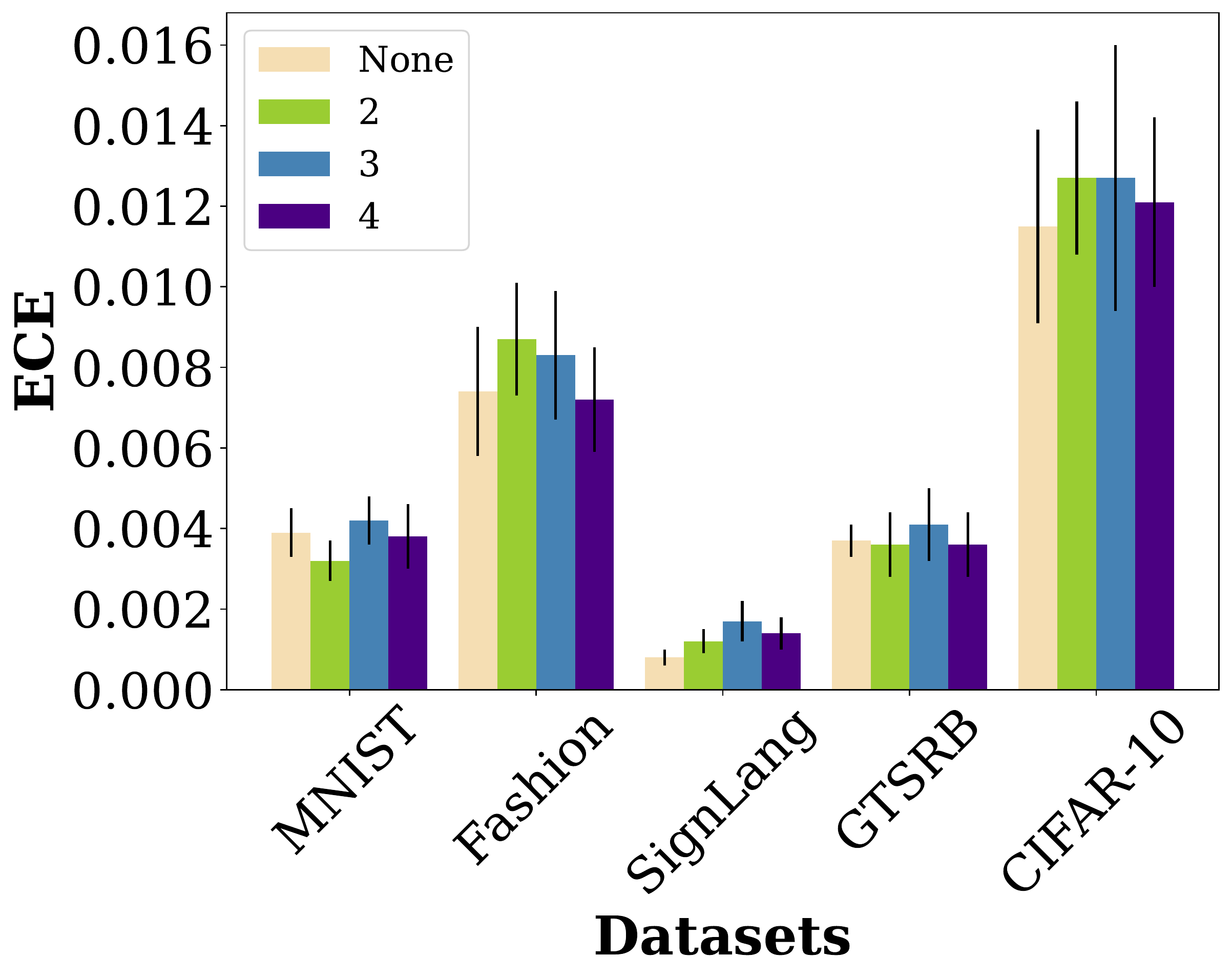}} &
  \subfloat[CNN]{\includegraphics[width=0.31\linewidth]{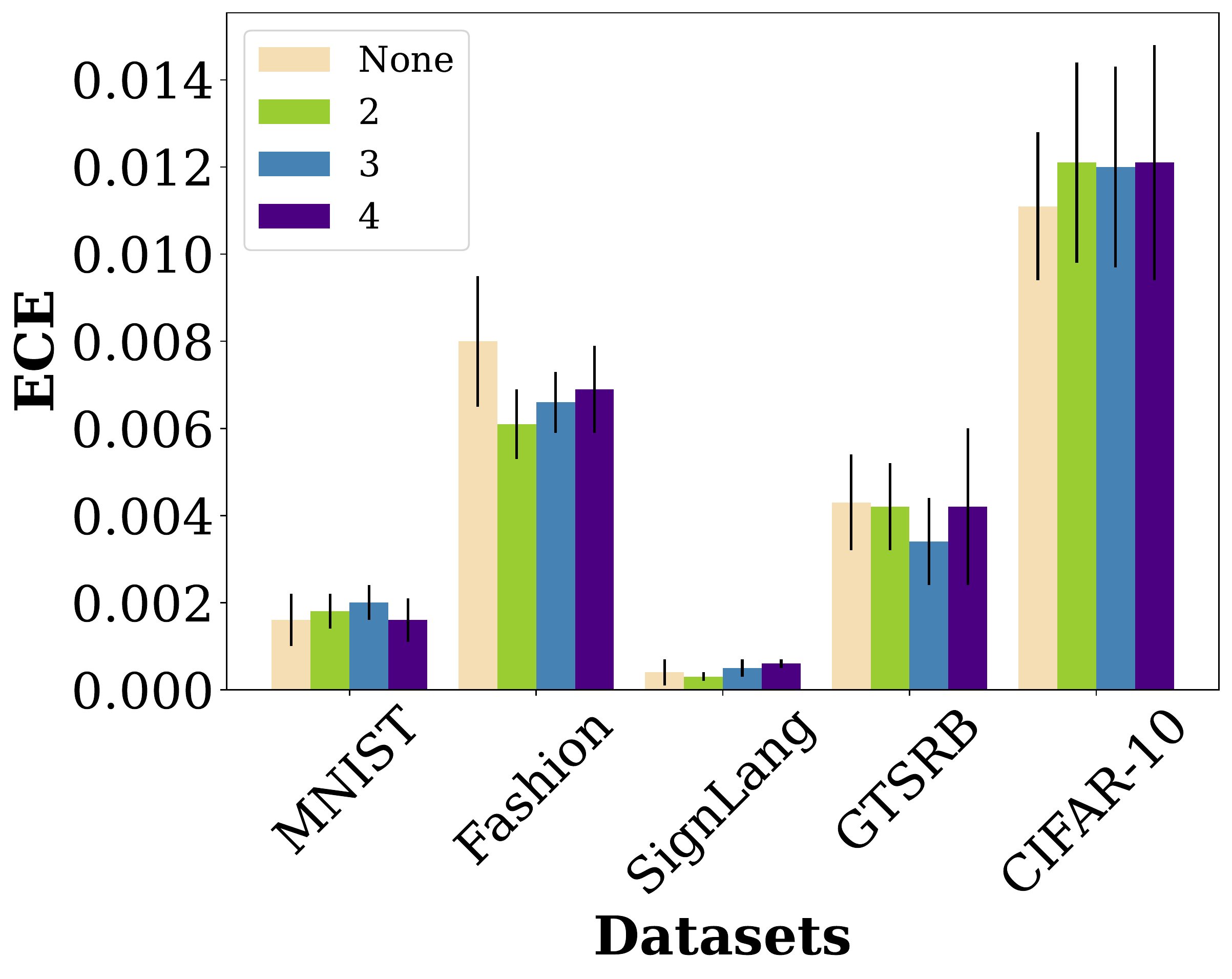} }
    \end{tabular}
    \caption{ECE measures with $95\%$ confidence intervals on 10 shuffles for various datasets on several models after applying max-pooling with different reduction parameters. }
    \label{fig:ECE-Maxpool}
\end{figure}

When using our method one needs to ship besides the model and the fitting function also the dataset itself, since the calculation of the fast-separation requires the calculation of distances to all images in the training set. This may imply memory overhead, which can be critical when using big datasets. 
Using a reduction of the dataset allows us to reduce the training set size needed to be shipped. As we have shown, 
the reduced dataset still obtains improved results, thus freeing memory usage of the algorithm.
%
Users can also predetermine the trade-off they would like between throughput or memory and ECE and adjust the reduction parameters accordingly.


\section{Conclusion}
Our work introduces geometric separation-based algorithms for confidence estimation in machine learning models. Specifically, we measure a geometric separation score and use the specific model to translate each score value into a confidence value using a standard post-hoc calibration method. Thus, inputs close to training set examples of the same class receive higher confidence than those close to examples with a different classification. Thus, our algorithms depend on the specific model but as a black box resulting in methods that work for all machine learning models.

Our evaluation shows that geometric separation improves confidence estimations in visual workloads. However, calculating geometric separation is computationally complex and time intensive. Thus, we suggest multiple approximation techniques to speed up the process and bring it to practicality. 
Our extensive evaluation shows that such approximations retain most of the benefits of geometric separations and drastically improve confidence estimation along with supporting many calculations per second, enabling real-time applications. For example, we can process live camera feeds at multiple hundreds of calculations per second.



Our work is unique because it extracts a new external signal to derive confidence estimations. Thus, we can leverage the existing post-hoc calibration techniques to calibrate our signal and meet various optimization criteria. 
We showed that the same calibration method (Isotonic regression) yields a lower ECE when performed on the geometric signal rather than on the model's original signal. The achieved accuracy improves on a diverse set of recently proposed calibration methods. Notably, our approach reduces the error in confidence estimations by up to 99\% compared to alternative methods (depending on the specific dataset and model).   

Looking into the future, we plan to address the dependence of this work on normalized inputs and tackle datasets with variable-sized images. In such datasets, the geometric distances may also depend on the resolution and alignment of the object. As a director, we plan to use the CNN middle layer latent space as the feeding vector (rather than the original images) in the geometric separation calculation. In any case, the ability to derive fast approximations of geometric separation would be a valuable tool in future research.  
\bibliography{Chouraqui_457}

\end{document}